\documentclass[11pt]{article}

\usepackage[final]{acl}

\usepackage{times}
\usepackage{latexsym}
\usepackage{amsmath,amssymb,amsthm,mathtools}
\newtheorem{proposition}{Proposition}

\theoremstyle{definition}

\usepackage{tabularx}
\usepackage[T1]{fontenc}
\usepackage[utf8]{inputenc}

\usepackage{microtype}

\usepackage{inconsolata}

\usepackage{graphicx}
\usepackage{algorithm}
\usepackage{algpseudocode}
\usepackage{booktabs}
\usepackage{multirow}
\usepackage{float}
\usepackage[most]{tcolorbox}
\usepackage{enumitem}
\usepackage{makecell}
\usepackage{nicematrix}
\usepackage[table]{xcolor}

\definecolor{clustergray}{gray}{0.93}

\title{Clustering-Based Balanced Sampling and Allocation with Data Parallelism for High-Performance Fine-Tuning}
\newcommand{\algname}{\textbf{CluSTER}}

\usepackage{caption}

\author{
 \textbf{Hyunjin Kim},
 \textbf{Youngeun Nam},
 \textbf{Jaemin Han},
 \textbf{Wonhyeok Choi},
 \textbf{Jae-Gil Lee\thanks{Corresponding Author.}}
\\
 KAIST
\\
 \texttt{\{hyunjin, youngeun.nam, ggp07150, wonhyeok316, jaegil\}@kaist.ac.kr}
}

\begin{document}
\maketitle
\begin{abstract}
% V.5
% Instruction-tuning datasets for large language models\,(LLMs) are often large, redundant, and imbalanced, limiting efficient adaptation. In na\"ive large-batch fine-tuning, especially with multi-GPU data parallelism\,(DP), overrepresented groups are repeatedly sampled while underrepresented but informative ones remain weakly covered. We propose \algname{}, a \textbf{Clu}ster-aware balanced \textbf{S}ampling framework for \textbf{T}raining \textbf{E}fficient data \textbf{R}eduction in DP instruction tuning. \algname{} curates a representative reduced dataset via gradient-space clustering and DP-aware balanced allocation, ensuring dual-level coverage across clusters and workers, while weighting preserves the original data distribution. This reduces redundant computation and improves training stability without compromising model quality. Across multiple instruction-tuning datasets, \algname{} reduces training time by 64.6\% with almost no accuracy loss over prior sampling and data reduction methods. Code is available at \url{https://anonymous.4open.science/r/CluSTER}.

Instruction-tuning datasets for large language models\,(LLMs) are often large, redundant, and imbalanced, limiting efficient adaptation. Naive large-batch fine-tuning repeatedly includes overrepresented sample groups while weakly covering underrepresented but informative ones, especially under data parallelism\,(DP) across multiple GPUs. We propose \algname{}, a \textbf{Clu}ster-aware balanced \textbf{S}ampling framework for \textbf{T}raining \textbf{E}fficient data \textbf{R}eduction in DP instruction tuning. \algname{} curates a representative reduced dataset through gradient-space clustering and DP-aware balanced allocation, ensuring dual-level coverage across clusters and workers, while preserving the original data distribution by weighted update. As a result, \algname{} reduces redundant computation and improves training stability without compromising model quality. Across multiple instruction-tuning datasets, \algname{} reduces training time by up to 69.6\% with almost no accuracy loss compared to prior sampling and data reduction methods. Code is available at \url{https://github.com/kaist-dmlab/CluSTER}.

\end{abstract}
\section{Introduction}
Fine-tuning is a key mechanism for adapting large language models\,(LLMs) to downstream tasks~\citep{wei2021finetuned}. In particular, instruction tuning has emerged as a standard fine-tuning paradigm for aligning pretrained models with task-specific behaviors and user-facing objectives~\citep{liu2023makes}. At the same time, instruction-tuning datasets are commonly large-scale and heterogeneous, with long-tailed distributions across topics, styles, and difficulty levels~\citep{henning2023survey}. Consequently, effective instruction tuning requires more than computational scaling; it also requires principled strategies for selecting and organizing informative training samples.

A key challenge emerges when large-scale datasets are used in large-batch fine-tuning, where each optimization step aggregates signals from many samples. Training efficiency depends not only on the number of processed samples, but also on whether they provide informative and complementary update directions~\citep{yin2018gradient}. Under random sampling, long-tailed instruction-tuning datasets tend to overrepresent dominant regions, leaving underrepresented yet informative samples insufficiently covered~\citep{li2024superfiltering, liu2023makes}. As a result, computation is spent on redundant training signals, reducing the diversity of useful supervision within each optimization step.

Large-batch LLM fine-tuning is commonly implemented through data parallelism\,(DP). Prior work mainly improves system-level efficiency, such as memory sharding and communication optimization~\citep{dean2012large, li2014parameter, shoeybi2019megatron, huang2019gpipe}, while leaving sample composition and assignment within large batches unaddressed. In DP, the global batch is split across workers, whose local gradients are aggregated into a single update. Thus, optimization depends on the diversity of gradients generated in parallel~\citep{yin2018gradient}. With long-tailed and redundant data, dominant sample groups may be assigned to multiple workers, leaving rarer but informative groups insufficiently covered. Figure~\ref{fig:overview}(a) shows this inefficiency: dominant groups (e.g., red and blue) are repeatedly processed across different workers, leading to redundant gradient computations without proportional gains in convergence.

\begin{figure}[t]
    \centering
    \includegraphics[width=0.495\textwidth, trim={3 7 0 5}, clip]{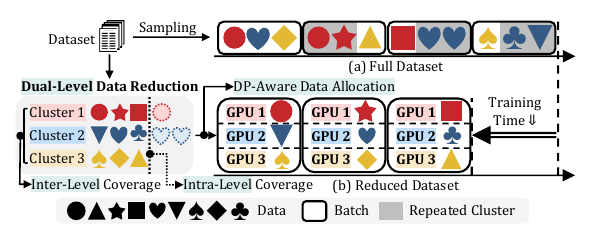} % figure_1_intro_cov.pdf
    \caption{
        Comparison between full- and reduced-dataset training. 
        Colors indicate clusters and shapes denote finer-grained sample types.
        (a)~With the full dataset, redundant and dominant patterns recur across batches and workers. 
        (b)~With the reduced dataset, diverse representative samples are selected with balanced cluster coverage and allocated across GPUs, reducing computation and training time while preserving the data distribution.
    }
    \label{fig:overview}
    \vspace{-0.5cm}
\end{figure}

Our goal is therefore to construct a \emph{reduced} dataset that improves large-batch fine-tuning efficiency by removing redundant and weakly informative samples while preserving full-dataset training behavior. Unlike prior data selection approaches that primarily rank samples by individual informativeness, we focus on how selected samples \emph{collectively} contribute to distributed mini-batch optimization. Thus, \emph{a reduced dataset should not only contain informative samples, but also organize them so that each optimization step covers diverse and complementary training signals}. We define a \emph{cluster} as a coherent sample group that provides a similar training signal. As shown in Figure~\ref{fig:overview}(b), this leads to two coverage requirements. First, across workers, batches should cover different sample groups rather than repeatedly drawing from the same group. Second, within each group, selected samples should remain diverse and complementary by covering different sub-clusters or sample types, represented by the distinct shapes in the figure. We refer to these requirements as \emph{dual-level coverage}: \emph{inter-level coverage} across workers and \emph{intra-level coverage} within clusters.

For this purpose, it is important to identify sample groups according to their roles in model optimization.
Since our objective is to approximate full-dataset optimization with fewer samples, surface-level similarity alone is insufficient: textually different samples can induce similar updates, while similar-looking samples can drive learning in different directions.
To better preserve the loss-surface dynamics of full-dataset training, samples should therefore be grouped by update similarity.
We thus organize samples in the \emph{gradient space}, where samples with similar gradients are expected to guide the model toward similar optimization directions.

We propose \algname{}, a \textbf{Clu}ster-aware balanced \textbf{S}ampling framework for \textbf{T}raining \textbf{E}fficient data \textbf{R}eduction in DP instruction tuning. \algname{} clusters samples in gradient space and builds a reduced dataset via balanced cluster sampling to reduce redundant signals across workers. It promotes \emph{dual-level coverage}: \emph{inter-level coverage} by assigning different clusters to workers, and \emph{intra-level coverage} by selecting diverse, boundary-near samples within clusters.
To reduce majority-group dominance, clusters are under-sampled to the smallest cluster size, while importance-weighted updates preserve the original data distribution.
Overall, \algname{} reduces redundant computation while preserving full-dataset optimization behavior.

\algname{} improves both training and selection preprocessing efficiency across multiple instruction-tuning datasets. Compared with full-dataset instruction tuning, \emph{it reduces training time by \textbf{69.6\%} with almost no degradation in model performance}. The gain comes from replacing repeated updates from redundant samples with more complementary updates across workers.
\section{Related Work}
\subsection{Data Sampling}
% Data sampling is an important factor for optimization stability and generalization in large-scale training.
Data sampling has been recognized as an important factor for optimization stability and generalization in large-scale training.
Random sampling is unbiased but can induce high gradient variance under imbalanced data distributions. Although uniform or stratified sampling mitigates this issue~\cite{loshchilov2016online, katharopoulos2018not}, it is often impractical for text or instruction datasets with imbalanced, dynamic, and implicit clusters.
% Data sampling has been recognized as an important factor for optimization stability and generalization in large-scale training.
% Random sampling provides unbiased gradient estimates but can lead to high gradient variance under highly imbalanced data distributions.
% Although uniform or stratified sampling methods have been proposed to ensure that mini-batches better reflect the global data distribution~\cite{loshchilov2016online, katharopoulos2018not}, such approaches are often impractical for text or instruction datasets, where data clusters are inherently imbalanced, dynamically structured, and not explicitly defined.

Recent works have explored sampling strategies that prioritize informative or diverse samples based on learning utility, including importance- or curriculum-based sampling~\cite{mindermann2022prioritized, ash2020deep}, difficulty-based instruction filtering such as IFD~\citep{li2024quantity}, and representation- or diversity-based selection in embedding or gradient space~\cite{sener2018active, coleman2020selection}.
For LLM fine-tuning, related approaches include LESS, which estimates sample influence~\cite{xia2024less}, and S2L, which uses small-model training trajectories to select data for larger models~\cite{yang2024smalltolarge}.
Although these approaches improve sample quality, many require costly preprocessing, such as per-example gradients, influence scores, or proxy-model trajectories, which is expensive for large datasets.
Unlike these example-level approaches, \algname{} uses target-model gradient-proxy clusters for joint DP-aware worker allocation and cluster-weighted updates during fine-tuning.

\subsection{Data Parallelism (DP)}
DP is a core paradigm for scaling large neural networks across multiple GPUs. In DP, each worker processes a different mini-batch, and gradients are aggregated into a global update. Over time, several architectures, including parameter server~\citep{li2014parameter}, ZeRO~\citep{rajbhandari2020zero}, and fully sharded data parallel\,(FSDP)~\citep{zhao2023pytorch}, have improved DP efficiency in terms of scalability, memory footprint, and communication overhead.
\section{Preliminary}

\label{sec:sampling-variance}
\begin{figure}[t]
    \centering
    \includegraphics[width=0.49\textwidth, trim={0 3 0 0}, clip]{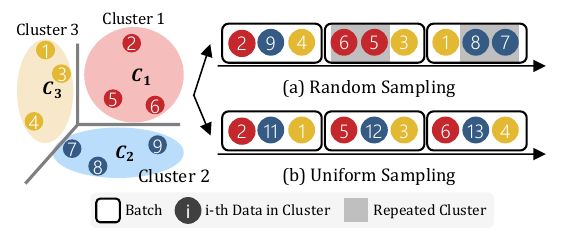}
    \caption{
    Comparison of random and uniform sampling under an idealized balanced-cluster setting.
    (a)~Random sampling frequently produces repeated semantic clusters across batches, while (b)~uniform sampling encourages balanced cluster coverage.
    }
    \label{fig:preliminary}
    \vspace{-0.5cm}
\end{figure}

Random sampling is widely used for mini-batch construction because it provides an unbiased gradient estimate under the independent and identically distributed\,(i.i.d.) assumption.
However, when the data distribution is heterogeneous or imbalanced, random sampling may over-sample dominant regions while under-sampling minority ones, resulting in high variance of mini-batch gradients.
As shown in Figure~\ref{fig:preliminary}, random sampling produces batches with markedly different data compositions, whereas uniform sampling maintains
consistent cluster distributions.
To isolate the intrinsic difference between sampling strategies, we begin by analyzing an idealized balanced scenario in which the dataset is evenly partitioned into homogeneous clusters; the comparison between the two strategies reveals the fundamental effect of sampling uniformity on gradient variance and stability.

Since the convergence rate of stochastic optimization is closely tied to the variance of gradient estimates~\citep{bottou2018optimization, gower2019sgd}, reducing the variance of mini-batch gradients directly contributes to faster and more stable convergence. Therefore, understanding how different sampling strategies influence gradient variance is critical to improving training efficiency.

\paragraph{Setup.}
We consider a dataset partitioned into $K$ clusters 
$\{C_c\}_{c=1}^K$, each containing $n$ samples so that the total dataset 
size is $N = Kn$. The gradient of each sample $i$ is denoted by 
$g_i \in \mathbb{R}^d$.
A summary of the notation used in the analysis is provided in Appendix~\ref{app:notation}.

\paragraph{Variance Decomposition.}
The total gradient variability can be decomposed into
\emph{intra-cluster} and \emph{inter-cluster} components.
Specifically, the average intra-cluster gradient variance across all clusters is defined as
\begin{equation}
    S^2_{\mathrm{intra}} = \frac{1}{K} \sum_{c=1}^K S_c^2,
\end{equation}
and the inter-cluster gradient variance, which measures the dispersion
of cluster-wise mean gradients around the population mean gradient, is
\begin{equation}
    % S^2_{\mathrm{inter}} = \frac{1}{K} \sum_{c=1}^K \|\mu_c - \mu\|^2.
    S_{\text{inter}}^2=\tfrac{1}{K}\sum_{c=1}^{K}\|\mu_c-\mu\|_2^2.
\end{equation}
A detailed derivation of this decomposition is provided in Appendix~\ref{app:prop1}.
% See Appendix~\ref{app:prop1} for detailed decomposition.
% The detailed decomposition is provided in Appendix~\ref{app:prop1}.

\paragraph{Mini-batch Gradient Estimator.}
For a batch of size $B$, the empirical gradient estimator is
\begin{equation}
    \hat{G} = \frac{1}{B} \sum_{i \in \mathrm{batch}} g_i,
\end{equation}
with expectation $\mathbb{E}[\hat{G}] = \mu$. 
We denote the sampling ratio as $f = B/N$.

\begin{proposition}[Variance of random sampling]
\label{prop:srs}
Under random sampling (without replacement),
\begin{align}
\mathrm{Var}(\hat G_{\mathrm{Random}})
&\simeq \tfrac{1-f}{B}\,S^2 \label{eq:srs-base}\\[-2pt]
&= \tfrac{1-f}{B}\big(S_{\text{intra}}^2+S_{\text{inter}}^2\big), \label{eq:srs-split}
\end{align}
where $S^2$ is the finite-population variance.
\end{proposition}
\vspace*{-0.2cm}
\begin{proof}
See Appendix~\ref{app:proofsrs}.
\end{proof}

\begin{proposition}[Variance of uniform sampling]
\label{prop:strat}
Sampling $b=B/K$ per cluster and averaging with $W_c=N_c/N=1/K$ yields
\begin{align}
\mathrm{Var}(\hat G_{\mathrm{strat}})
&= \sum_{c=1}^K \tfrac{W_c^2}{b}\,(1\!-\!f_c)\,S_c^2 \label{eq:strat-gen}\\[-2pt]
% &= \tfrac{1-f}{B}\cdot \tfrac{1}{K}\sum_{c=1}^K S_c^2\\[-2pt]
&= \tfrac{1-f}{B}\,S_{\text{intra}}^2, \label{eq:strat-final}
\end{align}
with $f_c=b/N_c=f$.
\end{proposition}
\vspace*{-0.2cm}
\begin{proof}
See Appendix~\ref{app:proofstrat}.
\end{proof}

\begin{figure*}[t]
    \centering
    \includegraphics[width=\textwidth, trim={0 38 0 0}, clip]{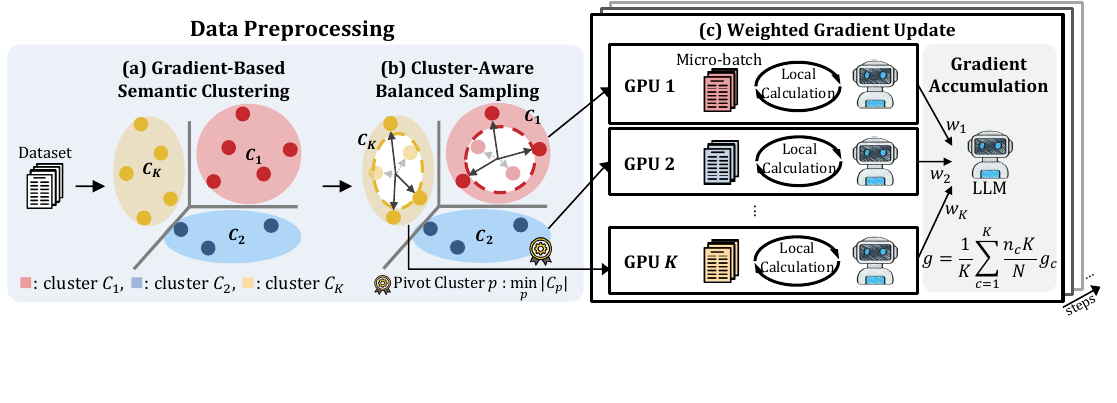}
    \caption{Overview of the proposed \algname{} framework.
    It consists of three stages: gradient-based clustering, cluster-aware balanced sampling, and weighted gradient update, which collectively ensure balanced data allocation.}
    \label{fig:method}
    \vspace*{-0.3cm}
\end{figure*}

\begin{proposition}[Variance comparison]
\label{prop:comparison}
Combining Propositions~\ref{prop:srs} and \ref{prop:strat}, if
$S_{\text{inter}}^2>0$, then
$\mathrm{Var}(\hat G_{\mathrm{strat}})<\mathrm{Var}(\hat G_{\mathrm{Random}})$,
with equality if and only if $S_{\text{inter}}^2=0$.
\end{proposition}
\vspace*{-0.2cm}
\begin{proof}
{See Appendix~\ref{app:prop3}}
\end{proof}

From Propositions~\ref{prop:srs} and \ref{prop:strat}, under the idealized balanced-cluster setting, uniform sampling reduces gradient variance by eliminating the inter-cluster variance component.
However, real-world datasets generally exhibit non-uniform cluster sizes, and the balanced analysis does not directly characterize the practical setting of \algname{}.
We therefore extend the analysis to weighted stratified sampling over non-uniform clusters and explicitly connect the resulting estimator to the cluster-weighted aggregation used in \algname{}.

\paragraph{Extension to Non-Uniform Clusters.}
For non-uniform clusters, let $\pi_c=|C_c|/N$ denote the original proportion of cluster $c$, and let $\hat{\mu}_c$ be the mean gradient of $m_c$ samples drawn from that cluster.
The weighted stratified estimator
\begin{equation}
    \hat{G}_{\mathrm{wstrat}}
    = \sum_{c=1}^{K}\pi_c\hat{\mu}_c
\end{equation}
remains unbiased under uniform within-cluster sampling. Unlike the balanced case, this formulation does not guarantee lower total variance than random sampling; rather, it removes variance due to random fluctuations in cluster composition.

For \algname{}, $m_c=B/K$ under the default allocation and
$w_c=K\pi_c$. Thus,
\begin{equation}
    \frac{1}{K}\sum_{c=1}^{K} w_c\hat{\mu}_c
    = \sum_{c=1}^{K}\pi_c\hat{\mu}_c,
\end{equation}
which corresponds to the weighted stratified estimator above.
This formulation provides the theoretical motivation for the balanced cluster allocation and weighted aggregation used in \algname{} under practical non-uniform cluster distributions.

% From Propositions~\ref{prop:srs} and \ref{prop:strat}, under the balanced assumption, uniform sampling reduces gradient variance by eliminating the inter-cluster variance component. However, real-world datasets are typically highly skewed across clusters.
% To address this, we design a cluster-aware sampling strategy that approximates the variance-reduction effect of uniform sampling by adaptively balancing the data distribution before sampling.

\section{\algname{} Framework}
We propose \algname{}, which balances data selection between clusters to mitigate overfitting while preserving the intrinsic data distribution. 
% \algname{} follows a three-stage process comprising gradient-based semantic clustering, cluster-aware balanced sampling, and weighted gradient update.
As shown in Figure~\ref{fig:method}, the method consists of three stages: gradient-based semantic clustering, cluster-aware balanced sampling, and weighted gradient update.
\subsection{Gradient-Based Semantic Clustering}
Inspired by BADGE~\citep{ash2020deep}, we construct per-sample \emph{gradient proxy embeddings} that capture both semantic information and model uncertainty. 
Computing gradients with respect to all model parameters is expensive for large language models, so we instead derive a lightweight proxy from the token-level cross-entropy loss and the final-layer hidden representations.
For a sequence $x$, let $h_t \in \mathbb{R}^{H}$ denote the final hidden state at token position $t$, and $p_{t,y_t}$ the teacher-forced probability of the ground-truth token $y_t$. 
We define a token-level gradient proxy as $g_t = (p_{t,y_t}-1)h_t$. 
The sequence-level embedding is obtained by averaging over valid token positions,
$g(x)=\frac{1}{|\mathcal{T}(x)|}\sum_{t\in\mathcal{T}(x)}(p_{t,y_t}-1)h_t$,
followed by $\ell_2$ normalization.
As the hidden states $h_t$ encode the semantic context of the sequence, the resulting proxy embedding inherits semantic structure from the representation space. 
The coefficient $(p_{t,y_t}-1)$ further weights tokens according to prediction confidence, allowing the embedding to capture both semantic content and uncertainty.

Finally, we cluster the normalized embeddings using $K$-means~\citep{lloyd1982least}, as illustrated in Figure~\ref{fig:method}(a).
Samples that induce similar update directions cluster together, yielding groups that reflect coherent instructional intents and reasoning patterns.
% The number of clusters $k$ is set equal to the number of GPUs to enable GPU-level data allocation that preserves intra-cluster consistency.
The number of clusters, $K$, is automatically set equal to the number of GPUs to enable direct cluster-to-GPU assignment during training that preserves intra-cluster consistency.
To prevent severe cluster imbalance, we apply capacity-constrained $K$-means with $|\mathcal{C}_c| \le \frac{\alpha|D|}{K}$, where $\alpha$ is a user-defined hyperparameter that controls the upper bound on the cluster capacity.

% \smallskip
% \noindent \textbf{Step 2: Cluster-Aware Balanced Sampling.}
\subsection{Cluster-Aware Balanced Sampling}
\label{sec:balanced_sampling}
% \begin{figure*}[t]
%     \centering
%     \includegraphics[width=0.8\textwidth]{figures/figure_3_selection.pdf}
%     \caption{Conceptual illustration of expected pairwise cosine similarities under 
%     different within-cluster selection strategies.
%     (a) Random selection draws points according to the underlying cluster distribution, resulting in an intermediate level of angular variability.
%     (b) Core-centric selection restricts samples to a small, dense region near the centroid, producing large cosine similarities and highly aligned gradient directions.
%     (c) Peripheral selection focuses on boundary points that span a wide angular region, yielding the smallest cosine similarities and the most diverse gradient directions. 
%     }
%     \label{fig:cos}
% \end{figure*}
%
% We construct a reduced dataset by balancing the sample budget across clusters\,(inter-level coverage) and selecting informative, non-redundant samples within each cluster\,(intra-level coverage). We implement inter-cluster balancing with a pivot-cluster budget and perform intra-cluster selection using distance-to-centroid ranking.
We construct a reduced dataset by balancing the sample budget evenly across clusters (inter-level coverage) and selecting non-redundant samples within each cluster (intra-level coverage).
For inter-cluster coverage, we use the smallest cluster as the pivot and down-sample larger clusters to match the size of the pivot cluster using distance-to-centroid ranking.
% Within each cluster, samples are ranked by their distance to the cluster centroid, and peripheral samples are selected to promote gradient diversity while reducing redundancy.
% After clustering, we balance the number of samples across clusters to stabilize training. We first identify the smallest cluster and down-sample the others to match its size, yielding a uniformly balanced dataset. To analyze the optimization effect of within-cluster selection, 
% We examine its influence on gradient diversity~\citep{yin2018gradient}.
% The gradient diversity of a mini-batch $\mathcal{S}$ is defined as
We examine how the intra-cluster selection strategy affects gradient diversity~\citep{yin2018gradient}.
The gradient diversity of a mini-batch $\mathcal{S}$ is defined as
\begin{equation}
\label{eq:diversity}
\begin{aligned}
\Delta(\mathcal{S})
&= \frac{\sum_{i\in\mathcal{S}}\|g_i\|^2}
{\left\|\sum_{i\in\mathcal{S}} g_i\right\|^2} \\
&= \frac{\sum_{i}\|g_i\|^2}
{\sum_i \|g_i\|^2
+ 2 \sum_{i<j} \|g_i\|\,\|g_j\| \cos\theta_{ij}}
\end{aligned}
\end{equation}
which decreases as pairwise cosine similarities increase.
Thus, the expected value of $\cos\theta_{ij}$ directly determines
the diversity level of a mini-batch.

% Within a cluster $\mathcal{C}_c$, we decompose each gradient as $g_i = \bar{\mu}_c + \varepsilon_i$, where $\bar{\mu}_c$ is the cluster-level mean direction and $\varepsilon_i$ captures within-cluster residual variation. Let $v=\bar{\mu}_c/\|\bar{\mu}_c\|$, and decompose the residual into the component parallel to $v$ and the orthogonal component: \begin{equation} \varepsilon_i = a_i v + z_i, \qquad z_i = (I-vv^\top)\varepsilon_i . \end{equation} The orthogonal residual $z_i$ captures the directional deviation of sample $i$ from the cluster-level gradient direction. For a selection strategy $\pi$, let $S_\pi$ denote the selected subset and define \begin{equation} \kappa_\pi = \mathbb{E}\left[\|z_i\|^2 \mid i\in S_\pi\right]. \end{equation} A second-order expansion of the normalized gradient $u_i=g_i/\|g_i\|$ gives \begin{equation} \label{eq:selected-cosine} \mathbb{E}\left[\cos\theta_{ij}\mid i,j\in S_\pi\right] \approx 1 - \frac{\kappa_\pi}{\|\bar{\mu}_c\|^2} + R_\pi , \end{equation} where \begin{equation} R_\pi = \frac{1}{\|\bar{\mu}_c\|^2} \mathbb{E}\left[ \langle z_i,z_j\rangle \mid i,j\in S_\pi \right] \end{equation} is a residual correlation term. When the selected residual directions are approximately centered and weakly correlated, $R_\pi$ is small. In this case, selection strategies that preserve larger orthogonal residual magnitude $\kappa_\pi$ yield lower expected cosine similarity and thus higher gradient diversity. A complete derivation and discussion of the required assumptions are provided in Appendix~\ref{appendix:prune_strategy}.

Within a cluster $\mathcal{C}_c$, we decompose gradients as
$g_i = \bar{\mu}_c + \varepsilon_i$,
where $\bar{\mu}_c$ is the cluster mean and $\varepsilon_i$
captures residual variation.
Under a first-order approximation of normalized gradients,
the expected cosine similarity between gradients satisfies
\begin{equation}
\mathbb{E}[\cos\theta_{ij}]
\approx
1 - \frac{(d-1)\sigma^2}{\|\bar{\mu}_c\|^2},
\end{equation}
which decreases monotonically with the residual variance $\sigma^2$.
A complete derivation of this approximation is provided in Appendix~\ref{appendix:prune_strategy}.

\begin{figure}[t]
    \centering
    \includegraphics[width=0.49\textwidth, trim={0 3 0 0}, clip]{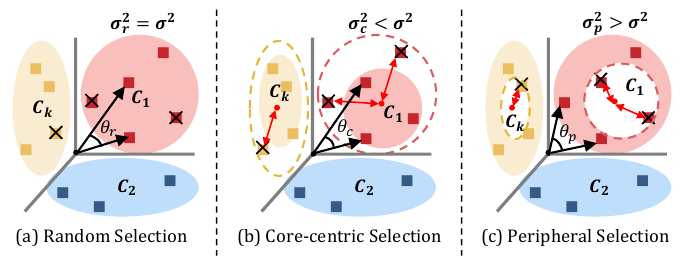}
    \caption{
    % Expected pairwise cosine similarities under different within-cluster selection strategies.
    Expected pairwise cosine similarity between gradient embeddings under within-cluster selection strategies:
    (a)~random sampling, (b)~core-centric sampling, and (c)~peripheral sampling.
    % (a) Random selection draws points according to the underlying cluster distribution, resulting in an intermediate level of angular variability.
    % (b) Core-centric selection restricts samples to a small, dense region near the centroid, producing large cosine similarities and highly aligned gradient directions.
    % (c) Peripheral selection focuses on boundary points that span a wide angular region, yielding the smallest cosine similarities and the most diverse gradient directions. 
    }
    \label{fig:cos}
    \vspace{-0.5cm}
\end{figure}

We consider three within-cluster selection strategies:
\textit{core-centric} (nearest to centroid), \textit{random}, and
\textit{peripheral} (farthest from centroid), as illustrated in Figure~\ref{fig:cos}.
Since distance-based conditioning induces
$\sigma_{\text{core}}^2<\sigma_{\text{rand}}^2<\sigma_{\text{peri}}^2$,
it also implies
$\Delta_{\text{core}}<\Delta_{\text{rand}}<\Delta_{\text{peri}}$
by Equation~\eqref{eq:diversity}.
As shown in Figure~\ref{fig:method}(b), we therefore adopt peripheral selection as the default strategy.
To construct the final training subset, we first down-sample all clusters to the smallest cluster size $n_{\min}$, yielding a uniformly balanced intermediate dataset.
Given a selection ratio $r \in (0,1]$, we retain only the top $r \cdot n_{\min}$ farthest samples from the centroid within each balanced cluster.
The resulting balanced subsets are then distributed across GPUs to encourage diverse and balanced gradient contributions.

\subsection{Weighted Gradient Update}
To ensure consistency with full-dataset fine-tuning, we introduce a weighting coefficient that aligns the expected gradient contribution of each cluster with that under full-dataset random sampling, as shown in Figure~\ref{fig:method}(c).
The derived coefficient for the \(c\)-th cluster is expressed as
$w_c = \frac{|\mathcal{C}_c| \, K}{N}$.
While balancing cluster sizes is effective for reducing sampling bias and stabilizing gradient estimation, it inevitably alters the original data distribution. 
To compensate for this shift, we assign each cluster a weight \(w_c\) proportional to its original size \(|\mathcal{C}_c|\), \emph{restoring its contribution} to the global gradient, maintaining consistency with the original unbalanced dataset distribution  while preserving the variance-reduction benefits of balanced sampling.
% The full derivation is provided in Appendix~\ref{appendix:weight-derivation}.
The full derivation is provided in Appendix~\ref{appendix:weight-derivation}, and the overall procedure of \algname{} is summarized in Appendix~\ref{appendix:alg}.

\smallskip
\noindent \textbf{Hyperparameters.}
\algname{} uses three hyperparameters: the number of clusters $K$, cluster capacity $\alpha$, and selection ratio $r$. By default, $K$ is set to the number of data-parallel workers, providing a simple granularity for cluster-balanced worker assignment. We set $\alpha=1.5$ and vary $r$ to control the final data ratio.
% \begin{table}[t]
% \centering
% \caption{
% Comparison of selection-time complexity. $D$ denotes the dataset size.
% \textit{Forward} and \textit{Backward} denote a single forward pass and backward pass of the transformer model, respectively.
% For a transformer model with $L$ layers, sequence length $T$, and hidden dimension $H$, a forward pass costs $O(L(T^2H + TH^2))$; a backward pass is approximately $2\times$ more expensive~\citep{vaswani2017attention, goodfellow2016deep}.
% }
% \label{tab:selection_complexity}
% \resizebox{0.75\columnwidth}{!}{
% \begin{tabular}{lc}
% \toprule
% Method & Compute Complexity \\
% \midrule
% IFD & $O(|D|\cdot \textit{Forward} \times 2)$ \\
% LESS & $O(|D|\cdot (\textit{Forward} + \textit{Backward}))$ \\
% S2L & $O(|D|\cdot (\textit{Forward} + \textit{Backward}))$ \\
% \algname{} & $O(|D|\cdot \textit{Forward})$ \\
% \bottomrule
% \end{tabular}
% }
% \vspace{-5mm}
% \end{table}

\begin{table}[t]
\centering
\caption{
Comparison of selection-time complexity. $D$ denotes the dataset size.
\textit{Forward} and \textit{Backward} denote a forward and backward pass of the target transformer model, respectively.
$\textit{Forward}_s$ and $\textit{Backward}_s$ denote those of the small model, and $E_s$ denotes the number of proxy-model training epochs.
For a transformer model with $L$ layers, sequence length $T$, and hidden dimension $H$, a forward pass costs $O(L(T^2H + TH^2))$; a backward pass is approximately $2\times$ more expensive~\citep{vaswani2017attention, goodfellow2016deep}.
}
\label{tab:selection_complexity}
\resizebox{0.75\columnwidth}{!}{
\begin{tabular}{lc}
\toprule
Method & Compute Complexity \\
\midrule
IFD & $O(|D|\cdot \textit{Forward} \times 2)$ \\
LESS & $O(|D|\cdot (\textit{Forward} + \textit{Backward}))$ \\
S2L & $O(|D|\cdot E_s \cdot (\textit{Forward}_s + \textit{Backward}_s))$ \\
\algname{} & $O(|D|\cdot \textit{Forward})$ \\
\bottomrule
\end{tabular}
}
% \vspace{-0.5cm}
\end{table}

% \subsection{Data Selection Complexity Analysis}
% \label{sec:complexity}
% We compare the computational complexity of \algname{} with existing data selection methods.
% Table~\ref{tab:selection_complexity} summarizes the theoretical complexity of representative approaches, including IFD and LESS~\citep{li2024quantity, xia2024less}.
% IFD evaluates difficulty by computing two likelihood scores per sample, corresponding to two forward passes over the dataset.
% LESS requires forward and backward passes to compute parameter gradients, resulting in a cost proportional to the number of model parameters.
% In contrast, \algname{} constructs lightweight gradient proxy embeddings using only the final-layer hidden states and token-level probabilities, avoiding expensive backpropagation over model parameters. 
% This leads to substantially lower computational complexity than existing approaches.

\subsection{Data Selection Complexity Analysis}
\label{sec:complexity}
We compare the computational complexity of \algname{} with existing data selection methods.
Table~\ref{tab:selection_complexity} summarizes the theoretical complexity of representative approaches, including IFD, LESS, and S2L~\citep{li2024quantity, xia2024less, yang2024smalltolarge}.
IFD evaluates difficulty by computing two likelihood scores per sample, corresponding to two forward passes over the dataset.
LESS requires forward and backward passes to compute parameter gradients, resulting in a cost proportional to the number of model parameters.
S2L uses small-model training trajectories, introducing additional proxy-model training and trajectory-tracking costs.
% In contrast, \algname{} constructs lightweight gradient-proxy embeddings using only the final-layer hidden states and token-level probabilities, avoiding expensive backpropagation or proxy-model training.
% This leads to substantially lower computational complexity than existing approaches.
In contrast, \algname{} constructs lightweight gradient-proxy embeddings from final-layer hidden states and token-level probabilities, which yields lower computational complexity than existing approaches by avoiding expensive backpropagation or proxy-model training.

% We compare the computational complexity of \algname{} with representative data selection methods, including IFD, LESS, and S2L~\citep{li2024quantity, xia2024less, yang2024smalltolarge}, as summarized in Table~\ref{tab:selection_complexity}.
% IFD requires two forward passes to compute per-sample likelihood scores, LESS computes parameter gradients via forward and backward passes, and S2L introduces proxy-model training and trajectory-tracking costs.
% In contrast, \algname{} constructs lightweight gradient-proxy embeddings from final-layer hidden states and token-level probabilities, which yields lower computational complexity than existing approaches.

\section{Experiments}
\begin{table}[t!]
\centering

\caption{
Pass@1\,(\%) and training time\,(hrs) for sampling strategies on \textrm{CodeLlama-7B} fine-tuned with \textrm{Magicoder-OSS-Instruct-75K}.
$r=\{1.00,0.75,0.50\}$ corresponds to 79\%, 59\%, and 40\% data usage after cluster balancing.
\textbf{Bold} marks the best score per usage.
}
\label{tab:main_magi}
\resizebox{\linewidth}{!}{
\setlength{\tabcolsep}{3pt}
\begin{tabular}{ll ccc c}
\toprule
\multicolumn{2}{c}{Data Selection} & \multicolumn{3}{c}{Evaluation Benchmark} & \multirow{2}{*}{\textbf{Time}} \\
\cmidrule(lr){1-2} \cmidrule(lr){3-5}
\textbf{Ratio} & \textbf{Sampling} & \textbf{HumanEval(+)} & \textbf{MBPP(+)} & \textbf{Avg.} & \\
\midrule
\multirow{2}{*}{\makecell[l]{Full}} 
  & Random  & 57.1{\scriptsize $\pm$ 1.5} (52.4{\scriptsize $\pm$ 2.7}) & 66.1{\scriptsize $\pm$ 0.6} (54.8{\scriptsize $\pm$ 0.6}) & 57.6 & 1.13 \\
  & Uniform & 54.1{\scriptsize $\pm$ 2.2} (50.0{\scriptsize $\pm$ 2.4}) & 65.0{\scriptsize $\pm$ 1.3} (54.3{\scriptsize $\pm$ 1.5}) & 55.8 & 1.29 \\
\midrule
\multirow{6}{*}{\makecell[l]{1.00\\(79\%)}} %No redundancy
  & Random  & \textbf{54.9{\scriptsize $\pm$ 1.6}} (49.4{\scriptsize $\pm$ 1.8}) & 65.9{\scriptsize $\pm$ 1.1} (54.7{\scriptsize $\pm$ 0.9}) & 56.2 & 0.90 \\
  & Uniform & 53.3{\scriptsize $\pm$ 2.3} (49.0{\scriptsize $\pm$ 1.7}) & 65.1{\scriptsize $\pm$ 0.3} (54.2{\scriptsize $\pm$ 0.7}) & 55.4 & 1.04 \\
  & IFD     & 53.8{\scriptsize $\pm$ 3.1} (49.8{\scriptsize $\pm$ 3.1}) & 65.3{\scriptsize $\pm$ 0.4} (54.0{\scriptsize $\pm$ 1.4}) & 55.7 & 0.87 \\
  & LESS    & 53.9{\scriptsize $\pm$ 2.2} (\textbf{50.0{\scriptsize $\pm$ 1.8}}) & 65.8{\scriptsize $\pm$ 0.2} (53.6{\scriptsize $\pm$ 0.7}) & 55.8 & 0.90\\
  & S2L    & 53.0{\scriptsize $\pm$ 1.1} (48.2{\scriptsize $\pm$ 0.7}) & 66.1{\scriptsize $\pm$ 0.9} (54.8{\scriptsize $\pm$ 1.1}) & 55.5 & 0.90\\
  & \cellcolor{clustergray}\textbf{\algname{}} 
             & \cellcolor{clustergray}\textbf{54.9{\scriptsize $\pm$ 1.2}} (49.8{\scriptsize $\pm$ 1.5}) 
             & \cellcolor{clustergray}\textbf{67.0{\scriptsize $\pm$ 1.3}} (\textbf{55.6{\scriptsize $\pm$ 1.1}}) 
             & \cellcolor{clustergray}\textbf{56.8} & \cellcolor{clustergray}0.92 \\
\midrule
\multirow{6}{*}{\makecell[l]{0.75\\(59\%)}} 
  & Random  & 51.4{\scriptsize $\pm$ 0.7} (46.8{\scriptsize $\pm$ 1.0}) & \textbf{67.2{\scriptsize $\pm$ 1.1}} (55.1{\scriptsize $\pm$ 0.7}) & 55.1 & 0.67 \\
  & Uniform & 52.6{\scriptsize $\pm$ 2.5} (48.0{\scriptsize $\pm$ 2.9}) & 66.1{\scriptsize $\pm$ 1.2} (54.4{\scriptsize $\pm$ 1.0}) & 55.3 & 0.77 \\
  & IFD     & 50.2{\scriptsize $\pm$ 0.7} (43.9{\scriptsize $\pm$ 2.1}) & 66.3{\scriptsize $\pm$ 0.4} (54.4{\scriptsize $\pm$ 0.4}) & 53.7 & 0.66 \\
  & LESS    & 53.3{\scriptsize $\pm$ 1.8} (49.0{\scriptsize $\pm$ 1.7}) & 65.4{\scriptsize $\pm$ 1.2} (53.8{\scriptsize $\pm$ 0.4}) & 55.4 & 0.67 \\
  & S2L     & 51.8{\scriptsize $\pm$ 1.4} (46.3{\scriptsize $\pm$ 1.7}) & 63.2{\scriptsize $\pm$ 1.2} (53.4{\scriptsize $\pm$ 0.9}) & 53.7 & 0.66 \\
  & \cellcolor{clustergray}\textbf{\algname{}} 
             & \cellcolor{clustergray}\textbf{54.5{\scriptsize $\pm$ 0.9} (49.4{\scriptsize $\pm$ 1.0})} 
             & \cellcolor{clustergray}66.9{\scriptsize $\pm$ 0.5} (\textbf{55.3{\scriptsize $\pm$ 0.3}}) 
             & \cellcolor{clustergray}\textbf{56.5} & \cellcolor{clustergray}0.69 \\
\midrule
\multirow{6}{*}{\makecell[l]{0.50\\(40\%)}} 
  & Random  & 52.2{\scriptsize $\pm$ 1.3} (45.9{\scriptsize $\pm$ 1.0}) & 66.0{\scriptsize $\pm$ 0.8} (54.7{\scriptsize $\pm$ 0.9}) & 54.7 & 0.45 \\
  & Uniform & 51.8{\scriptsize $\pm$ 2.7} (46.7{\scriptsize $\pm$ 2.3}) & 64.9{\scriptsize $\pm$ 1.7} (53.7{\scriptsize $\pm$ 1.1}) & 54.3 & 0.53 \\
  & IFD     & 50.6{\scriptsize $\pm$ 0.0} (44.7{\scriptsize $\pm$ 0.3}) & 64.5{\scriptsize $\pm$ 1.2} (53.8{\scriptsize $\pm$ 1.0}) & 53.4 & 0.44 \\
  & LESS    & 53.7{\scriptsize $\pm$ 1.1} (48.2{\scriptsize $\pm$ 1.2}) & 66.2{\scriptsize $\pm$ 1.2} (54.6{\scriptsize $\pm$ 2.2}) & 55.7 & 0.45 \\
  & S2L    & 52.4{\scriptsize $\pm$ 0.5} (45.7{\scriptsize $\pm$ 0.9}) & \textbf{66.7{\scriptsize $\pm$ 0.7}} (54.8{\scriptsize $\pm$ 1.3}) & 54.9 & 0.46 \\
  & \cellcolor{clustergray}\textbf{\algname{}} 
             & \cellcolor{clustergray}\textbf{54.0{\scriptsize $\pm$ 1.8} (49.0{\scriptsize $\pm$ 2.5})} 
             & \cellcolor{clustergray}65.8{\scriptsize $\pm$ 1.1} (\textbf{55.0{\scriptsize $\pm$ 0.9}}) 
             & \cellcolor{clustergray}\textbf{55.9} & \cellcolor{clustergray}0.47 \\
\bottomrule
\end{tabular}}

\vspace{0.7em}
\caption{
Accuracy\,(\%) on medical benchmarks and training time\,(hrs) for sampling strategies on Llama-2-7B fine-tuned with MedInstruct-52K. 
The data ratios 63\%, 47\%, and 31\% correspond to $r = 1.00$, $0.75$, and $0.50$.
}
\label{tab:main_med}
\resizebox{\linewidth}{!}{
\setlength{\tabcolsep}{3pt}
\begin{tabular}{ll ccccc c}
\toprule
\multicolumn{2}{c}{Data Selection} & \multicolumn{5}{c}{Evaluation Benchmark} & \multirow{2}{*}{\textbf{Time}} \\
\cmidrule(lr){1-2} \cmidrule(lr){3-7}
\textbf{Ratio} & \textbf{Sampling} & \textbf{MedMCQA} & \textbf{MedQA} & \textbf{MMLU} & \textbf{PubMedQA} & \textbf{Avg.} & \\
\midrule
\multirow{2}{*}{\makecell[l]{Full}}
  & Random  & 36.9 {\scriptsize $\pm$ 1.4} & 35.6 {\scriptsize $\pm$ 1.1} & 50.4 {\scriptsize $\pm$ 0.8} & 71.2 {\scriptsize $\pm$ 0.2} & 48.5 & 0.99 \\
  & Uniform & 37.4 {\scriptsize $\pm$ 0.4} & 36.6 {\scriptsize $\pm$ 0.2} & 50.3 {\scriptsize $\pm$ 1.4} & 69.0 {\scriptsize $\pm$ 1.0} & 48.3 & 1.02 \\
\midrule
\multirow{6}{*}{\makecell[l]{1.00\\(63\%)}} 
  & Random  & 36.7 {\scriptsize $\pm$ 0.6} & 35.3 {\scriptsize $\pm$ 1.8} & 48.9 {\scriptsize $\pm$ 0.4} & 70.1 {\scriptsize $\pm$ 0.9} & 47.7 & 0.59 \\
  & Uniform & 37.2 {\scriptsize $\pm$ 0.4} &  \textbf{36.3 {\scriptsize $\pm$ 1.2}} &  \textbf{49.8 {\scriptsize $\pm$ 1.3}} &  69.8 {\scriptsize $\pm$ 0.3} & 48.3 & 0.67 \\
  & IFD     & 36.0 {\scriptsize $\pm$ 0.6} & 35.6 {\scriptsize $\pm$ 0.8} & 47.6 {\scriptsize $\pm$ 0.7} & 69.9 {\scriptsize $\pm$ 0.4} & 47.3 & 0.59 \\
  & LESS    & 36.3 {\scriptsize $\pm$ 0.3} & 33.8 {\scriptsize $\pm$ 1.2} & 49.5 {\scriptsize $\pm$ 1.0} & \textbf{72.7 {\scriptsize $\pm$ 0.2}} & 48.1 & 0.59 \\
  & S2L    &  36.9 {\scriptsize $\pm$ 0.4} &  35.0 {\scriptsize $\pm$ 1.7} & \textbf{49.8 {\scriptsize $\pm$ 1.1}} & 69.0 {\scriptsize $\pm$ 1.7} & 47.7 & 0.57 \\
  & \cellcolor{clustergray}\textbf{\algname{}} 
             & \cellcolor{clustergray}\textbf{37.5 {\scriptsize $\pm$ 0.2}} & \cellcolor{clustergray}\textbf{36.3 {\scriptsize $\pm$ 0.9}} & \cellcolor{clustergray}\textbf{49.8 {\scriptsize $\pm$ 1.2}} & \cellcolor{clustergray}70.1 {\scriptsize $\pm$ 1.3} & \cellcolor{clustergray}\textbf{48.4} & \cellcolor{clustergray}0.59 \\
\midrule
\multirow{6}{*}{\makecell[l]{0.75\\(47\%)}} 
  & Random  & 36.7 {\scriptsize $\pm$ 1.3} & 34.6 {\scriptsize $\pm$ 2.8} & 48.3 {\scriptsize $\pm$ 1.2} & 70.9 {\scriptsize $\pm$ 1.4} & 47.6 & 0.45 \\
  & Uniform & 36.4 {\scriptsize $\pm$ 0.2} & 35.3 {\scriptsize $\pm$ 0.3} & 48.5 {\scriptsize $\pm$ 0.8} & 69.7 {\scriptsize $\pm$ 0.4} & 47.5 & 0.52 \\
  & IFD     & 35.6 {\scriptsize $\pm$ 0.4} & 34.1 {\scriptsize $\pm$ 0.3} & 44.6 {\scriptsize $\pm$ 0.5} & 66.3 {\scriptsize $\pm$ 0.9} & 45.2 & 0.46 \\
  & LESS    & 35.7 {\scriptsize $\pm$ 0.3} & 32.9 {\scriptsize $\pm$ 1.6} & 49.2 {\scriptsize $\pm$ 1.2} & \textbf{72.9 {\scriptsize $\pm$ 0.2}} & 47.7 & 0.46 \\
  & S2L    &  37.0 {\scriptsize $\pm$ 1.1} &  34.8 {\scriptsize $\pm$ 3.8} &  47.0 {\scriptsize $\pm$ 1.9} &69.6 {\scriptsize $\pm$ 0.9} & 47.1 & 0.45 \\
  & \cellcolor{clustergray}\textbf{\algname{}} 
             & \cellcolor{clustergray}\textbf{37.1 {\scriptsize $\pm$ 0.2}} & \cellcolor{clustergray}\textbf{36.2 {\scriptsize $\pm$ 0.3}} & \cellcolor{clustergray}\textbf{49.3 {\scriptsize $\pm$ 1.7}} & \cellcolor{clustergray}70.3 {\scriptsize $\pm$ 1.3} & \cellcolor{clustergray}\textbf{48.2} & \cellcolor{clustergray}0.45 \\
\midrule
\multirow{6}{*}{\makecell[l]{0.50\\(31\%)}} 
  & Random  & 36.7 {\scriptsize $\pm$ 1.6} & 34.4 {\scriptsize $\pm$ 0.8} & 48.9 {\scriptsize $\pm$ 1.6} & 69.2 {\scriptsize $\pm$ 1.6} & 47.3& 0.32 \\
  & Uniform & 36.1 {\scriptsize $\pm$ 0.5} & \textbf{36.3 {\scriptsize $\pm$ 0.7}} & 48.4 {\scriptsize $\pm$ 3.3} & 69.3 {\scriptsize $\pm$ 1.0} & 47.5 & 0.36  \\
  & IFD     & 36.0 {\scriptsize $\pm$ 0.2} & 36.1 {\scriptsize $\pm$ 1.3} & 45.9 {\scriptsize $\pm$ 0.8} & 66.5 {\scriptsize $\pm$ 0.9} & 46.1 & 0.32 \\
  & LESS    & 35.5 {\scriptsize $\pm$ 0.8} & 34.1 {\scriptsize $\pm$ 0.9} & 47.2 {\scriptsize $\pm$ 1.8} & \textbf{72.2 {\scriptsize $\pm$ 0.7}} & 47.2 & 0.32 \\
  & S2L    &  36.1 {\scriptsize $\pm$ 0.9} &  35.0 {\scriptsize $\pm$ 2.5} &  48.6 {\scriptsize $\pm$ 3.6} & 68.6 {\scriptsize $\pm$ 0.9} & 47.1 & 0.31 \\
  & \cellcolor{clustergray}\textbf{\algname{}} 
             & \cellcolor{clustergray}\textbf{37.4 {\scriptsize $\pm$ 0.3}} & \cellcolor{clustergray}35.3 {\scriptsize $\pm$ 2.6} & \cellcolor{clustergray}\textbf{49.0 {\scriptsize $\pm$ 1.7}} & \cellcolor{clustergray}70.4 {\scriptsize $\pm$ 1.5} & \cellcolor{clustergray}\textbf{48.0} & \cellcolor{clustergray}0.30 \\
\bottomrule
\end{tabular}
}
\end{table}

\subsection{Experimental Setups}

\noindent \textbf{Baselines.}
We compare \algname{} with two categories of baselines:
(1) \textbf{sampling-based methods}, including \textbf{random sampling}, the default strategy in standard DP, and 
\textbf{uniform sampling}~\cite{loshchilov2016online}, an idealized setting where data are evenly distributed across clusters; and
(2) \textbf{selection-based methods}, including \textbf{IFD sampling}~\citep{li2024quantity}, which is based on the instruction-following difficulty,
\textbf{LESS sampling}~\cite{xia2024less}, which prioritizes data based on gradient-based influence estimation, and
\textbf{S2L}~\cite{yang2024smalltolarge}, which selects data by clustering small-model training trajectories.

\smallskip
\noindent \textbf{Datasets.}
For fine-tuning, we use the \textit{code instruction datasets} \textrm{Magicoder-OSS-Instruct-75K}~\citep{wei2024magicoder} and \textrm{Evol-Instruct-Code-80K}~\citep{luo2024wizardcoder}, 
as well as the \textit{medical dataset} \textrm{MedInstruct-52K}~\citep{zhang2023alpacare}. 
We first prune clusters to the pivot size and then apply sampling ratios of 1.00, 0.75, and 0.50.

\smallskip
\noindent \textbf{Models.}
We evaluate three LLMs: \textrm{CodeLlama-Python-7B} and \textrm{CodeLlama-Python-13B}~\citep{roziere2023code} on code datasets, and \textrm{Llama-2-7B}~\citep{touvron2023llama} on the medical dataset with four H100 GPUs, except for Section~\ref{sec:gpu_scale}, which uses eight H100 GPUs.

\smallskip
\noindent \textbf{Benchmarks.}
% For code models, we report pass@1 with greedy decoding on HumanEval~\citep{chen2021evaluating}, HumanEval+~\citep{liu2023your}, MBPP~\citep{austin2021program}, and MBPP+~\citep{liu2023your}.
The trained code models are evaluated on four code domain benchmarks: HumanEval~\citep{chen2021evaluating}, HumanEval+~\citep{liu2023your}, MBPP~\citep{austin2021program} and MBPP+~\citep{liu2023your}.
For code benchmarks, we use greedy decoding to generate a single sample per problem, with results reported using pass@1.
For medical models, we evaluate performance using accuracy on four medical QA benchmarks: MedMCQA~\cite{pal2022medmcqa}, MedQA~\cite{jin2021disease}, PubMedQA~\cite{jin2019pubmedqa}, and MMLU medical subsets~\cite{hendrycks2021measuring}.

\smallskip
\noindent \textbf{Implementation Details.}
For \textrm{Magicoder-OSS-Instruct-75K}, we fine-tune \textrm{CodeLlama-Python-7B} for 2 epochs on four NVIDIA H100 GPUs using PyTorch DDP and Adafactor with a learning rate of $5e{-5}$. The global batch size is 512.
For \textrm{AlpaCare-MedInstruct-52K}, we fine-tune \textrm{Llama-2-7B} for 3 epochs on four H100 GPUs using FSDP and AdamW with a learning rate of $2e{-5}$.
We report results averaged over three random seeds and set the cluster capacity hyperparameter to $\alpha=1.5$.
All experiments are repeated \emph{three} times.

See Appendix~\ref{app:exp_set_app} for more detailed setups.

\subsection{Main Results}
\subsubsection{Results on Code Generation}
% Table~\ref{tab:main_magi} compares \algname{} with random, uniform, IFD, LESS, and S2L sampling under varying data budgets.
% \algname{} consistently achieves the best overall performance while substantially reducing training time.
% The model reaches an average score of 56.5 using only 59\% of the data, close to the full-data random baseline of 57.6, and remains competitive even at 40\% data usage with up to 58\% training-time reduction.
% These results demonstrate that our method improves data efficiency by preserving an informative subset of the training data.
In Table~\ref{tab:main_magi}, we compare \algname{} with random, uniform, IFD, and LESS sampling under varying data budgets.
\algname{} consistently achieves the best overall performance while also reducing training time.
By selectively removing less informative samples while preserving the overall cluster structure, \algname{} improves data efficiency without sacrificing performance.
For instance, with only 59\% of the data, it achieves an average score of 56.5, which is close to the full-data random baseline (57.6).
Even when the data usage is reduced to 40\%, it maintains competitive performance while reducing training time by up to 58\%.
These results demonstrate that \algname{} effectively preserves informative samples, maintaining performance with markedly fewer training samples.

\subsubsection{Results on Medical QA} % : Medical Domain / (Medical).
Table~\ref{tab:main_med} compares \algname{} with several baselines in the medical domain.
To assess whether the proposed strategy generalizes beyond the code domain, we fine-tune the model on a medical dataset and compare it with existing sampling methods.
The results show that \algname{} consistently outperforms all baselines in this setting.
Moreover, compared with full-dataset fine-tuning, our method exhibits the smallest performance drop while using a reduced training set, indicating that it preserves informative samples effectively even under aggressive data reduction.
Even with only 31\% of the full dataset, the model maintains performance comparable to full-dataset training while reducing the overall training time by 69.6\%.
These results suggest that the proposed method generalizes well across domains and provides stable performance even under substantial data reduction.
%Table~\ref{tab:main_med} evaluates \algname{} on the medical domain to assess its generalization beyond code datasets.
% \algname{} consistently outperforms all baselines and shows the smallest performance drop compared with full-dataset fine-tuning, achieving comparable performance with only 31\% of the data while reducing training time by 69.6\%.
% These results suggest that the proposed method generalizes well across domains and provides stable performance even under substantial data reduction.

% \begin{table}[t]
% \centering
% \caption{Comparison of data selection and total time\,(hrs) across baselines under $r=1.0$ and $r=0.75$.}
% \label{tab:selection_time}
% \resizebox{\columnwidth}{!}{
% \setlength{\tabcolsep}{0pt}
% \begin{tabular}{lcccccc}
% \toprule
% \multirow{2}{*}{\textbf{Method}}  & \multicolumn{3}{c}{\textbf{\textrm{Magicoder-OSS-Instruct-75K}}} & \multicolumn{3}{c}{\textbf{\textrm{MedInstruct-52K}}} \\
% \cmidrule(lr){2-4} \cmidrule(lr){5-7}
% & Selection & \textbf{$r=1$} & \textbf{$r=0.75$} & Selection & \textbf{$r=1$} & \textbf{$r=0.75$} \\
% \midrule
% IFD & 1.62 & 2.49 & 2.28 & 1.01 & 1.60 & 1.47 \\
% LESS & 2.79 & 3.69 & 3.46 & 1.67 & 2.26 & 2.13 \\
% S2L & 0.95 & 1.85 & 1.61 & 0.72 & 1.29 &  1.17 \\
% \textbf{\algname{}} & \textbf{0.19} & \textbf{1.11} & \textbf{0.88} & \textbf{0.13} & \textbf{0.72} & \textbf{0.58} \\
% \bottomrule
% \end{tabular}
% }
% \vspace{-0.3cm}
% \end{table}

\begin{table}[t]
\centering
\caption{Comparison of data selection, training, and total model building time\,(hrs) across baselines.}
\label{tab:selection_time}
\resizebox{\columnwidth}{!}{
\begin{tabular}{ll ccc ccc}
\toprule
\multicolumn{2}{c}{\multirow{2}{*}[-0.5ex]{Data Selection}}
& \multicolumn{6}{c}{Time} \\
\cmidrule(lr){3-8}
\multicolumn{2}{c}{}
& \multicolumn{3}{c}{\textrm{CodeLlama-7B}}
& \multicolumn{3}{c}{\textrm{Llama-2-7B}} \\
\cmidrule(lr){1-2} \cmidrule(lr){3-5} \cmidrule(lr){6-8}
\textbf{Ratio} & \textbf{Sampling} & \textbf{Sel.} & \textbf{Train} & \textbf{Total} & \textbf{Sel.} & \textbf{Train} & \textbf{Total} \\
\midrule
\multirow{4}{*}{1.00} & 
    IFD & 1.62 & 0.87 & 2.49 & 1.01 & 0.59 & 1.60 \\
    & LESS & 2.79 & 0.90 & 3.69 & 1.67 & 0.59 & 2.26 \\
    & S2L & 0.95 & 0.90 & 1.85 & 0.72 & 0.57 &  1.29 \\
    & \cellcolor{clustergray}\textbf{\algname{}} & \cellcolor{clustergray}\textbf{0.19} & \cellcolor{clustergray}0.92 & \cellcolor{clustergray}\textbf{1.11} & \cellcolor{clustergray}\textbf{0.13} & \cellcolor{clustergray}0.59 & \cellcolor{clustergray}\textbf{0.72} \\
\midrule
\multirow{4}{*}{0.75} & 
    IFD & 1.62 & 0.66 & 2.28 & 1.01 & 0.46 & 1.47 \\
    & LESS & 2.79 & 0.67 & 3.46 & 1.67 & 0.46 & 2.13 \\
    & S2L & 0.95 & 0.66 & 1.61 & 0.72 & 0.45 & 1.17 \\
    & \cellcolor{clustergray}\textbf{\algname{}} & \cellcolor{clustergray}\textbf{0.19} & \cellcolor{clustergray}0.69 & \cellcolor{clustergray}\textbf{0.88} & \cellcolor{clustergray}\textbf{0.13} & \cellcolor{clustergray}0.45 & \cellcolor{clustergray}\textbf{0.58} \\
\bottomrule
\end{tabular}
}
% \vspace{-0.5cm}
\end{table}

\subsubsection{Analysis of Selection Strategy Efficiency}
Table~\ref{tab:selection_time} compares the selection time of different data selection methods.
\algname{} consistently achieves the shortest selection time among the compared methods.
Across both code and medical-domain datasets, our method substantially reduces data selection time compared with LESS, IFD, and S2L.
% This efficiency arises because \algname{} requires only one forward pass, while other baselines require expensive computations for data selection, as analyzed in Section~\ref{sec:complexity}.
This efficiency arises because \algname{} requires only one forward pass, while IFD requires two forward passes, LESS requires backward passes, and S2L incurs additional overhead from small-model training trajectory collection, as analyzed in Section~\ref{sec:complexity}.

% \begin{table}[t]
% \centering
% \caption{Performance comparison of random sampling and \algname{} at selection ratio 1.00\,(21\%) on 8 GPUs.
% }
% \label{tab:gpu_scale}
% \begin{minipage}{0.49\textwidth}
% \centering
% \resizebox{1.0\columnwidth}{!}{
% \begin{tabular}{lcccc}
% \toprule
% \textbf{Method} & \textbf{HumanEval(+)} & \textbf{MBPP(+)} & \textbf{Average} & \textbf{$\Delta$ Score (\%)} \\
% \midrule
% % \multirow{2}{*}{\makecell[l]{1.00\\(20.8\%)}} & 
%     Random & 50.1(44.3) & 65.2(53.2) & 53.3 & 7.52 \\
% \algname{} & \textbf{55.7(49.4)} & \textbf{64.6(53.3)} & \textbf{57.8} & \textbf{3.20} \\
% \bottomrule
% \end{tabular}}
% \end{minipage}
% \end{table}

\begin{table}[t]
\centering
\caption{Performance comparison of random sampling and \algname{} on 8 GPUs with CodeLlama-Python-7B trained on Magicoder-OSS-Instruct-75K.
% at selection ratio 1.00\,(21\%) 
}
\label{tab:gpu_scale}
\begin{minipage}{0.49\textwidth}
\centering
\resizebox{1.0\columnwidth}{!}{
\begin{tabular}{llccc}
\toprule
\textbf{Sel. Ratio} & \textbf{Sampling} & \textbf{HumanEval(+)} & \textbf{MBPP(+)} & \textbf{Avg.} \\
\midrule
Full (100\%) & Random & 56.7 (51.8) & 66.3 (55.3) & 57.5 \\
\midrule
\multirow{2}{*}{\makecell[l]{1.00 (41.6\%)}} & 
    Random & 50.1 (44.3) & \textbf{65.2} (53.2) & 53.3 \\
    & \cellcolor{clustergray}\algname{} & \cellcolor{clustergray}\textbf{55.7 (49.4)} & \cellcolor{clustergray}64.6 (\textbf{53.3}) & \cellcolor{clustergray}\textbf{55.8}\\
\midrule
\multirow{2}{*}{\makecell[l]{0.75 (31.2\%)}} & 
    Random & 52.7 (45.8) & 64.1 (53.3) & 54.0 \\
    & \cellcolor{clustergray}\algname{} & \cellcolor{clustergray}\textbf{53.4 (47.6)} & \cellcolor{clustergray}\textbf{65.1 (54.0)} & \cellcolor{clustergray}\textbf{55.0} \\
% \midrule
% \multirow{2}{*}{\makecell[l]{0.50 (10.4\%)}} & 
%     Random & 51.6(45.5) & \textbf{65.0}(53.4) & 53.9 \\
%     & \algname{} & \textbf{53.4(46.3)} & 64.6(\textbf{53.3}) & \textbf{54.4} \\
\bottomrule
\end{tabular}}
\end{minipage}
\vspace{-0.1cm}
\end{table}

\subsection{Analysis of GPU Scalability}
\label{sec:gpu_scale}
Table~\ref{tab:gpu_scale} repeats part of the experiments in Table~\ref{tab:main_magi}, scaling from four GPUs to eight GPUs. \algname{} continues to outperform the baseline when using more GPUs.
Despite using only 41.6\% of the original dataset, \algname{} achieves performance comparable to, and slightly better than, full-data training while clearly outperforming random sampling at the same data ratio.
A similar trend is observed at the 31.2\% data budget, where \algname{} consistently improves over random sampling.
These results indicate that our cluster-aware data selection scales effectively with an increasing number of GPUs.
% Table~\ref{tab:gpu_scale} repeats part of the experiments in Table~\ref{tab:main_magi}, scaling from four to eight GPUs. \algname{} continues to outperform the baseline while using more GPUs.
% Despite using only 41.6\% of the full dataset, \algname{} achieves performance comparable to full-data training while clearly outperforming random sampling at the same data ratio.
% A similar trend is observed at the 31.2\% data budget, where \algname{} consistently improves over random sampling.
% These results indicate that our method scales effectively with an increasing number of GPUs.

\begin{figure}[t]
    \centering
    \includegraphics[width=0.49\textwidth, trim={0 0 0 0}, clip]{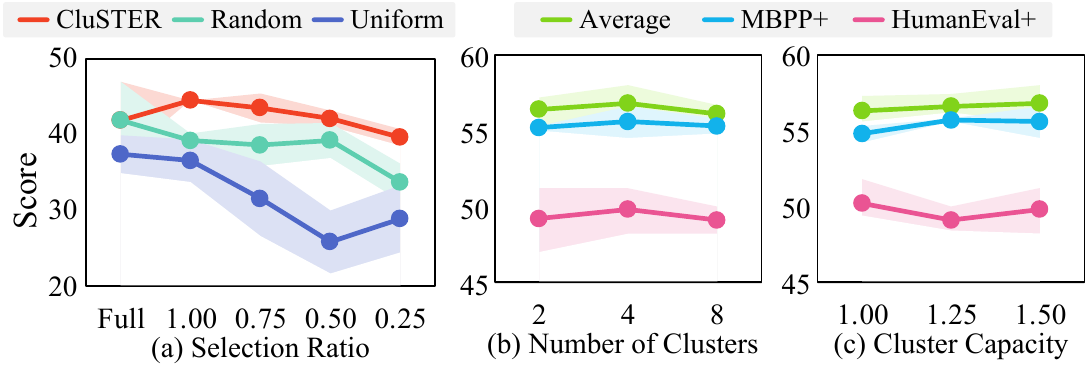}
    \caption{Hyperparameter sensitivity analysis of \algname{} with respect to (a) the selection ratio $r$, (b) the number of clusters $K$, and (c) the cluster capacity $\alpha$.}
    \label{fig:sensitivity_analysis}
    \vspace{-0.3cm}
\end{figure}

\subsection{Hyperparameter Sensitivity Analysis}
Figure~\ref{fig:sensitivity_analysis} summarizes the sensitivity of \algname{} to its key hyperparameters, selection ratio $r$, number of clusters $K$, and cluster capacity $\alpha$.

\subsubsection{Selection Ratio $r$}
To assess the effect of the final data budget, we vary the selection ratio $r$ after cluster balancing.
This experiment examines whether the selected subset continues to preserve useful training signals as the effective data usage decreases.
Figure~\ref{fig:sensitivity_analysis}\,(a) shows how the HumanEval+ score changes as the selection ratio decreases on \textrm{Evol-Instruct-Code-80K}.
Our approach degrades gradually and remains strong at low ratios, whereas random and uniform sampling drop sharply, widening the performance gap.
This suggests that \algname{} preserves informative and diverse training signals even when only a small fraction of the balanced data is used.

\subsubsection{Number of Clusters $K$}
We examine the sensitivity of \algname{} to the number of clusters $K$, which controls the granularity of the gradient-proxy partition.
This analysis tests whether the method remains stable when update-similarity groups are defined more coarsely or more finely.
Figure~\ref{fig:sensitivity_analysis}\,(b) compares different values of $K$.
Although $K$ defaults to the number of data-parallel workers, \algname{} remains robust across alternative values, suggesting that it is not sensitive to the exact clustering granularity.
This indicates that the benefit of the cluster-aware sampling framework does not come from a narrowly tuned cluster count, but from maintaining balanced coverage over diverse gradient-proxy regions.

\subsubsection{Cluster Capacity $\alpha$}
We analyze the role of the cluster capacity $\alpha$, which determines the maximum allowed cluster size in capacity-constrained clustering.
By changing $\alpha$, we examine whether performance is affected by making the clusters more balanced or allowing them to better follow the original data distribution.
Figure~\ref{fig:sensitivity_analysis}\,(c) varies the cluster capacity $\alpha$.
The performance remains stable across different values, indicating that the default setting $\alpha=1.5$ provides a reasonable balance between cluster balancing and distribution preservation.
This stability suggests that our framework is not overly dependent on a precise capacity constraint, as long as the clustering avoids highly imbalanced partitions.

Overall, these sensitivity results show that \algname{} remains stable under changes in hyperparmeters.
% the final data budget, clustering granularity, and capacity constraint.
This indicates that the method does not rely on a narrowly tuned hyperparameter configuration.

\begin{table}[t]
\centering
\caption{
Pass@1 of sampling strategies for \textrm{CodeLlama-Python-13B} on \textrm{Magicoder-OSS-Instruct-75K}.}
\label{tab:main_13b}
\resizebox{\linewidth}{!}{
\begin{tabular}{ll ccc}
\toprule
\textbf{Sel. Ratio} & \textbf{Sampling} & \textbf{HumanEval(+)} & \textbf{MBPP(+)} & \textbf{Avg.} \\
\midrule
\multirow{1}{*}{\makecell[l]{Full (100\%)}}
  & Random  & 63.4 (54.3) & 69.6 (57.8) & 61.3 \\
\midrule
\multirow{2}{*}{\makecell[l]{1.00 (81.5\%)}}
  & Random  & 62.2 (53.7) & 70.1 (58.4) & 61.1 \\
  & \cellcolor{clustergray}\textbf{\algname{}} 
             & \cellcolor{clustergray}\textbf{62.4} \textbf{(54.1)} 
             & \cellcolor{clustergray}\textbf{70.9 (58.7)} 
             & \cellcolor{clustergray}\textbf{61.5} \\ % 61.6 (53.7)	70.6 (59.0)	61.2
\midrule
\multirow{2}{*}{\makecell[l]{0.75 (61.1\%)}}
  & Random  & 62.8 (\textbf{54.6}) & 70.7 (59.0) & 61.8 \\ 
  & \cellcolor{clustergray}\textbf{\algname{}} 
             & \cellcolor{clustergray}\textbf{63.8} (53.2) 
             & \cellcolor{clustergray}\textbf{72.1} (\textbf{60.8}) 
             & \cellcolor{clustergray}\textbf{62.5} \\ 
\bottomrule
\end{tabular}}
\vspace{-0.3cm}
\end{table}
\subsection{Analysis of Model Scalability}
To examine whether the effectiveness of \algname{} transfers beyond the default 7B-scale setting, we further evaluate it on a larger \textrm{CodeLlama-Python-13B} model fine-tuned with \textrm{Magicoder-OSS-Instruct-75K}.
This experiment tests whether the gradient-proxy cluster structure remains useful as the target model capacity increases.
The results in Table~\ref{tab:main_13b} show that our method remains effective at the 13B scale, achieving competitive performance while using a reduced portion of the training data.
We further verify the effectiveness of \algname{} at the 13B scale on a distinct medical instruction-tuning setting; detailed results are provided in Appendix~\ref{appendix:additional_13b}.
This suggests that \algname{} is not limited to the 7B setting and can support data-efficient fine-tuning for larger models.

\begin{table}[t]
\centering
\caption{Effect of the weighting mechanism across different data ratios on \textrm{Magicoder-OSS-Instruct-75K}. %models and datasets
% Removing the weighting results in consistent performance degradation, highlighting its contribution to stable learning.
}
\label{tab:ablation_weight}
\resizebox{1\columnwidth}{!}{
\begin{NiceTabular}{l c ccc}
\toprule
\textbf{Sel. Ratio} & \textbf{Method} & \textbf{HumanEval(+)} & \textbf{MBPP(+)} & \textbf{Avg.} \\
% \midrule
% \multicolumn{5}{c}{\textbf{CodeLlama-Python-7B} fine-tuned on \textbf{Evol-Instruct-Code-80K}} \\
% \midrule
% \multirow{2}{*}{1.00 (68\%)} &
% w/ weighting & \textbf{50.2(44.5)} 
%              & \textbf{62.5(52.1)} & \textbf{52.3} \\
% & w/o weighting  & 47.3(41.5) & 56.1(46.9) & 48.0 \\
% \midrule
% \multirow{2}{*}{0.75 (51\%)} &
% w/ weighting & \textbf{48.8(43.5)} 
%              & \textbf{62.6(52.0)} & \textbf{51.7} \\
% & w/o weighting  & 47.2(40.1) & 61.8(50.2) & 49.8 \\
\midrule
% \multicolumn{5}{c}{\textbf{CodeLlama-Python-7B} fine-tuned on \textbf{Magicoder-OSS-Instruct-75K}} \\
% \midrule
\multirow{2}{*}{1.00 (79\%)} & w/o weighting  & 52.3(47.6) & 66.8(55.3) & 55.5 \\
& \textbf{w/ weighting} & \textbf{54.9(49.8)} & \textbf{67.0(55.6)} & \textbf{56.8} \\
\midrule
\multirow{2}{*}{0.75 (59\%)} &
w/o weighting  & 52.4(47.4) & 66.3(55.2) & 55.3 \\
& \textbf{w/ weighting} & \textbf{54.5(49.4)} & \textbf{66.9(55.3)} & \textbf{56.5} \\
\bottomrule
\end{NiceTabular}
}
\vspace{-0.3cm}
\end{table}

\subsection{Effect of Weighting Mechanism}
As shown in Table~\ref{tab:ablation_weight}, the proposed weighting mechanism consistently improves overall performance across different data ratios.
Across all settings, the weighted variant achieves higher scores, demonstrating improved robustness under varying data budgets.
% While minor fluctuations are observed on individual benchmarks, the overall trend remains consistent: weighting leads to stronger performance on HumanEval-style tasks and improves the aggregated average.
These results suggest that weighting samples according to their cluster sizes helps preserve the original contribution of each cluster during optimization.
This balancing effect leads to more stable performance across different data ratios.

\begin{table}[t]
\centering
\caption{
% Effect of the data selection mechanisms: random, core-centric, and peripheral selection.
Effect of data selection mechanisms on benchmark performance and gradient diversity\,(GD) using \textrm{Magicoder-OSS-Instruct-75K}.
}
\label{tab:ablation_prune}
\resizebox{1\columnwidth}{!}{
\begin{tabular}{l l cccc}
\toprule
\textbf{Sel. Ratio} & \textbf{Sel. Metric} & \textbf{GD} & \textbf{HumanEval(+)} & \textbf{MBPP(+)} & \textbf{Avg.} \\
\midrule
% \multicolumn{5}{c}{\textbf{CodeLlama-Python-7B} fine-tuned on \textbf{Evol-Instruct-Code-80K}} \\
% \midrule
% \multirow{3}{*}{1.00 (68\%)} &
% Random & 47.2(42.9) & 62.7(53.1) & 51.5 \\
% & Core-centric  & 45.7(41.9) & 62.1(51.8) & 50.4 \\
% & Peripheral  & \textbf{50.2(44.6)} & \textbf{62.5(52.1)} & \textbf{52.3} \\
% \midrule
% \multicolumn{5}{c}{\textbf{CodeLlama-Python-7B} fine-tuned on \textbf{Magicoder-OSS-Instruct-75K}} \\
% \midrule
\multirow{3}{*}{1.00 (79\%)} &
Random & 1.57 & 54.0 (48.8) & 66.4 (55.7) & 56.2 \\
& Core-centric & 1.51 & 54.3 (48.5) & 66.8 (54.8) & 56.1 \\
& \textbf{Peripheral} & \textbf{1.62} & \textbf{54.9(49.8)} 
             & \textbf{67.0(55.6)} 
             & \textbf{56.8} \\
\midrule
\multirow{3}{*}{0.75 (59\%)} &
Random & 1.58 & 53.6 (48.4) & 66.7 (55.5) & 56.1 \\
& Core-centric & 1.46 & 53.5 (47.0) & 66.0 (54.9) & 55.3 \\
& \textbf{Peripheral} & \textbf{1.67} & \textbf{54.5(49.4)} 
             & \textbf{66.9(55.3)} 
             & \textbf{56.5} \\
\bottomrule
\end{tabular}}
\vspace{-0.3cm}
\end{table}

\subsection{Analysis of Intra-Level Coverage}
After cluster-level balancing secures coverage over diverse gradient-space regions, the within-cluster selector determines which samples represent each cluster.
To evaluate the effect of different selection criteria, we compare three strategies introduced in Section~\ref{sec:balanced_sampling}: random, core-centric, and peripheral selection.
As shown in Table~\ref{tab:ablation_prune}, the results exhibit a clear ordering: peripheral selection performs best across most data ratios, followed by random selection, while core-centric selection generally yields the weakest results.
A similar ordering is observed in gradient diversity, where peripheral selection achieves the highest value, followed by random and core-centric selection.
This suggests that peripheral samples better preserve intra-cluster diversity, whereas core-centric samples mostly represent common gradients.
Thus, once clustering ensures inter-cluster diversity, selecting peripheral samples further improves intra-cluster coverage by retaining diverse samples within each cluster.
Overall, these results show that peripheral selection improves both downstream performance and gradient diversity, making it a more effective within-cluster selection strategy for \algname{}.

\begin{figure}[t]
    \centering
    \includegraphics[width=0.49\textwidth, trim={15 8 0 18}, clip]{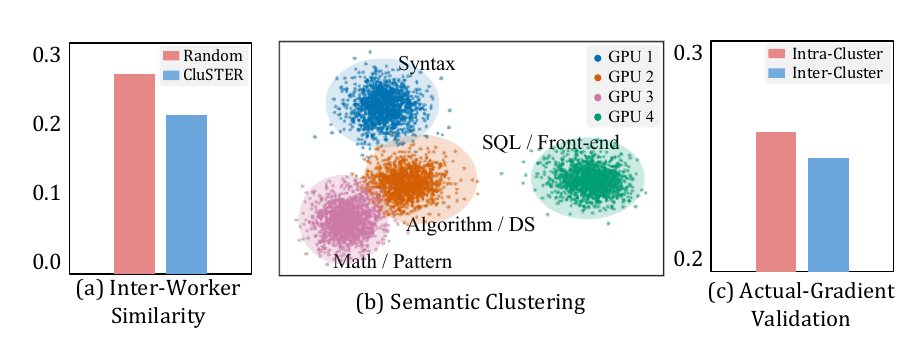}
    \caption{Analysis of inter-worker gradient redundancy and cluster-aware GPU allocation.
    (a) Inter-worker gradient cosine similarity.
    % (b) Gradient diversity.
    (b) Visualization of semantic clustering on a 5\% subset of \textrm{Magicoder-OSS-Instruct-75K}, where colors denote GPU assignments and text labels indicate dominant task categories.
    (c) Actual-gradient cosine similarity for intra- and inter-cluster sample pairs, where clusters are constructed using the proposed gradient proxies.
    }
    \label{fig:cos_sim}
    \vspace{-0.3cm}
\end{figure}

\subsection{Analysis of Inter-Level Coverage}
We examine whether \algname{} reduces gradient redundancy while preserving meaningful inter-level coverage.
Figure~\ref{fig:cos_sim}\,(a) reports the average inter-worker gradient cosine similarity over training.
\algname{} shows lower similarity than random sampling, indicating less redundant worker-level gradient signals and more diverse local updates across workers, thereby improving the diversity of gradient information in the global update.
% We examine whether \algname{} reduces gradient redundancy while preserving meaningful inter-level coverage.
% Figure~\ref{fig:cos_sim}\,(a) reports the average inter-worker gradient cosine similarity over training.
% \algname{} shows lower similarity than random sampling, indicating less redundant worker-level gradient signals and more diverse local updates across workers.

Figure~\ref{fig:cos_sim}\,(b) further visualizes the projected gradient embeddings.
The samples form distinct clusters whose dominant labels align with recognizable task categories, such as algorithm/data structure, SQL/front-end, mathematical reasoning, and basic syntax.
The dominant cluster categories show an 82\% agreement with manual annotations, suggesting that gradient-space clustering captures task-level semantics.

Figure~\ref{fig:cos_sim}\,(c) further validates the proxy clusters using actual per-sample gradients.
Intra-cluster pairs show higher average gradient cosine similarity than inter-cluster pairs, with a statistically significant difference under a cluster-size-preserving permutation test ($p=0.0464$).
Together, these results indicate that \algname{} improves inter-level coverage by reducing gradient redundancy and preserving task-level diversity in data-parallel training.

% \subsection{Analysis of Inter-Level Coverage}
% We examine whether the gradient-space clustering in \algname{} provides meaningful inter-level coverage.
% As shown in Figure~\ref{fig:cos_sim}, the samples form distinct clusters in the projected gradient space.
% A qualitative inspection shows that these clusters largely correspond to task categories, including algorithm and data structure problems, SQL and front-end code generation tasks, mathematical reasoning and pattern-matching instructions, and basic programming grammar or syntax exercises.
% The dominant cluster categories show an 82\% agreement with manually identified instruction categories, suggesting that the gradient-space clustering captures task-level instructional semantics.

% Figure~\ref{fig:cos_sim} further shows that \algname{} assigns samples to GPUs in a more cluster-aware manner than random sampling.
% Accordingly, our method yields lower inter-worker gradient cosine similarity, indicating less redundant local updates across GPUs.
% These results demonstrate that \algname{} improves inter-level coverage by preserving task-level diversity and reducing gradient redundancy in data-parallel training.

See Appendix~\ref{app:exp_app} for more experimental results.
\bigskip
\section{Conclusion}
In this work, we studied how sampling affects efficiency and stability in DP fine-tuning under redundant and imbalanced instruction data.
We proposed \algname{}, which combines gradient-space clustering, cluster-aware balanced sampling, and weighted updates to reduce redundant computation and stabilize multi-GPU optimization.
% which combines gradient-based clustering, cluster-aware balanced sampling, and cluster-weighted updates to improve multi-GPU batch balance while preserving the original data distribution.
% we investigated the impact of data sampling strategies on the efficiency and stability of large-scale language model training under data parallelism.
% We showed that random sampling, though widely used, introduces semantic redundancy and imbalance across batches, leading to high gradient variance and inefficient optimization.
% To address these limitations, we proposed \algname{}, a cluster-aware balanced sampling framework that balances data allocation through gradient-based clustering and cluster-weighted updates.
% By improving batch balance across GPUs while preserving the original data distribution, \algname{} enables faster training with competitive performance under reduced data budgets.
By improving batch balance across GPUs while preserving the original data distribution, \algname{} enables faster training with competitive performance under reduced data budgets, while also reducing the preprocessing time required for data selection.
% through its lightweight selection procedure.
% effectively reduces gradient variance while preserving the intrinsic data distribution, resulting in faster convergence with improved training efficiency.
% Extensive experiments on multiple large language models and instruction-tuning datasets demonstrated that \algname{} achieves comparable or superior performance to full-dataset fine-tuning while improving training efficiency.
These results highlight the importance of structured data sampling in distributed optimization and suggest a new direction for variance-aware data parallelism.
\newpage
\section*{Limitations}

While our experiments demonstrate the effectiveness of the proposed sampling strategy under multi-GPU data-parallel training, several limitations remain. First, our evaluation is conducted within a single-node multi-GPU setting. Although we set a large training scale up to eight GPUs with a batch size of 512 and demonstrated consistent improvements, we do not extensively explore multi-node environments with larger batch sizes and higher communication overhead.
Second, due to computational constraints, our experiments are limited to moderately sized instruction-tuning datasets, and we do not evaluate the method on substantially larger corpora containing hundreds of thousands or millions of samples. Investigating the behavior of the proposed approach under significantly larger batch regimes and substantially larger datasets remains an important direction for future work.

\section*{Ethics Statement}

This work uses existing pretrained models, training datasets, evaluation benchmarks, and baseline methods only for research purposes. 
We cite the creators of all artifacts used in our experiments, including pretrained models, instruction-tuning datasets, evaluation benchmarks, and baseline methods. 
The training datasets are used only for studying data selection and instruction tuning, while the evaluation benchmarks are used only for measuring downstream task performance. 
We use publicly available datasets and benchmarks and do not collect new user data. 
We rely on the preprocessing and release procedures of the dataset creators for handling personally identifying or offensive content, and we do not attempt to recover or expose such information. 
Pretrained models are used only for research-purpose fine-tuning. 
Our use of these artifacts is consistent with their intended research use and original access conditions.

We reviewed the licenses and terms of use of the artifacts used in this work. 
The pretrained models, including \textrm{CodeLlama-Python-7B}, \textrm{CodeLlama-Python-13B}, \textrm{Llama-2-7B}, and \textrm{Llama-2-13B} are used under the Llama 2 Community License. 
For training datasets, \textrm{Magicoder-OSS-Instruct-75K} is released under the MIT license, \textrm{Evol-Instruct-Code-80K} under the Creative Commons Attribution-NonCommercial-ShareAlike 4.0 International license, and \textrm{MedInstruct-52K} under the Creative Commons Attribution-NonCommercial 4.0 International license. 
For evaluation benchmarks, HumanEval, MedQA, PubMedQA, and MMLU are used under the MIT license; HumanEval+ and MBPP+ under the Apache-2.0 license; MBPP under the CC BY 4.0 license; and MedMCQA under the MIT license.

We do not redistribute original datasets, benchmark data, or pretrained model weights. 
Artifacts created in this work are intended only for research, reproduction, and analysis, and are not used outside the research context. 
Although our method does not introduce a user-facing application, fine-tuned models may inherit limitations from pretrained models and training datasets and should be evaluated before high-stakes use.

\section*{Acknowledgments}

This work was supported by Institute of Information \& Communications Technology Planning \& Evaluation\,(IITP) grant funded by the Korea government\,(MSIT) (No.\ RS-2020-II200862, DB4DL: High-Usability and Performance In-Memory Distributed DBMS for Deep Learning, 50\% and No.\ RS-2022-II220157, Robust, Fair, Extensible Data-Centric Continual Learning, 40\%) and Artificial intelligence industrial convergence cluster development project funded by the Ministry of Science and ICT\,(MSIT, Korea) \& Gwangju Metropolitan City (AICA-26-AE000008, 10\%).

\bibliography{custom}

\begin{thebibliography}{38}
\providecommand{\natexlab}[1]{#1}

\bibitem[{Ash et~al.(2020)Ash, Zhang, Krishnamurthy, Langford, and Agarwal}]{ash2020deep}
Jordan~T. Ash, Chicheng Zhang, Akshay Krishnamurthy, John Langford, and Alekh Agarwal. 2020.
\newblock \href {https://openreview.net/forum?id=ryghZJBKPS} {Deep batch active learning by diverse, uncertain gradient lower bounds}.
\newblock In \emph{Proceedings of the International Conference on Learning Representations (ICLR)}.

\bibitem[{Austin et~al.(2021)Austin, Odena, Nye, Bosma, Michalewski, Dohan, Jiang, Cai, Terry, Le, and Sutton}]{austin2021program}
Jacob Austin, Augustus Odena, Maxwell Nye, Maarten Bosma, Henryk Michalewski, David Dohan, Ellen Jiang, Carrie Cai, Michael Terry, Quoc Le, and Charles Sutton. 2021.
\newblock Program synthesis with large language models.
\newblock \emph{arXiv preprint arXiv:2108.07732}.

\bibitem[{Bottou et~al.(2018)Bottou, Curtis, and Nocedal}]{bottou2018optimization}
L{\'e}on Bottou, Frank~E Curtis, and Jorge Nocedal. 2018.
\newblock Optimization methods for large-scale machine learning.
\newblock \emph{SIAM Review}, 60(2):223--311.

\bibitem[{Chen et~al.(2021)Chen, Tworek, Jun, Yuan, de~Oliveira~Pinto, Kaplan, Edwards, Burda, Joseph, Brockman, Ray, Puri, Krueger, Petrov, Khlaaf, Sastry, Mishkin, Chan, Gray, Ryder, Pavlov, Power, Kaiser, Bavarian, Winter, Tillet, Such, Cummings, Plappert, Chantzis, Barnes, Herbert{-}Voss, Guss, Nichol, Paino, Tezak, Tang, Babuschkin, Balaji, Jain, Saunders, Hesse, Carr, Leike, Achiam, Misra, Morikawa, Radford, Knight, Brundage, Murati, Mayer, Welinder, McGrew, Amodei, McCandlish, Sutskever, and Zaremba}]{chen2021evaluating}
Mark Chen, Jerry Tworek, Heewoo Jun, Qiming Yuan, Henrique~Pond{\'{e}} de~Oliveira~Pinto, Jared Kaplan, Harri Edwards, Yuri Burda, Nicholas Joseph, Greg Brockman, Alex Ray, Raul Puri, Gretchen Krueger, Michael Petrov, Heidy Khlaaf, Girish Sastry, Pamela Mishkin, Brooke Chan, Scott Gray, and 39 others. 2021.
\newblock Evaluating large language models trained on code.
\newblock \emph{arXiv preprint arXiv:2107.03374}.

\bibitem[{Coleman et~al.(2020)Coleman, Yeh, Mussmann, Mirzasoleiman, Bailis, Liang, Leskovec, and Zaharia}]{coleman2020selection}
Cody Coleman, Christopher Yeh, Stephen Mussmann, Baharan Mirzasoleiman, Peter Bailis, Percy Liang, Jure Leskovec, and Matei Zaharia. 2020.
\newblock \href {https://arxiv.org/abs/1906.11829} {Selection via proxy: Efficient data selection for deep learning}.
\newblock In \emph{Proceedings of the International Conference on Learning Representations (ICLR)}.

\bibitem[{Dean et~al.(2012)Dean, Corrado, Monga, Chen, Devin, Mao, Ranzato, Senior, Tucker, Yang, Le, and Ng}]{dean2012large}
Jeffrey Dean, Greg Corrado, Rajat Monga, Kai Chen, Matthieu Devin, Mark Mao, Marc\textquotesingle~aurelio Ranzato, Andrew Senior, Paul Tucker, Ke~Yang, Quoc Le, and Andrew Ng. 2012.
\newblock Large scale distributed deep networks.
\newblock \emph{Proceedings of the Conference on Neural Information Processing Systems (NeurIPS)}, 25.

\bibitem[{Goodfellow et~al.(2016)Goodfellow, Bengio, and Courville}]{goodfellow2016deep}
Ian Goodfellow, Yoshua Bengio, and Aaron Courville. 2016.
\newblock \emph{Deep Learning}.
\newblock MIT Press.

\bibitem[{Gower et~al.(2019)Gower, Loizou, Qian, Sailanbayev, Shulgin, and Richt{\'a}rik}]{gower2019sgd}
Robert~Mansel Gower, Nicolas Loizou, Xun Qian, Alibek Sailanbayev, Egor Shulgin, and Peter Richt{\'a}rik. 2019.
\newblock Sgd: General analysis and improved rates.
\newblock In \emph{Proceedings of the International Conference on Machine Learning (ICML)}, pages 5200--5209. PMLR.

\bibitem[{Hendrycks et~al.(2021)Hendrycks, Burns, Basart, Zou, Mazeika, Song, and Steinhardt}]{hendrycks2021measuring}
Dan Hendrycks, Collin Burns, Steven Basart, Andy Zou, Mantas Mazeika, Dawn Song, and Jacob Steinhardt. 2021.
\newblock \href {https://openreview.net/forum?id=d7KBjmI3GmQ} {Measuring massive multitask language understanding}.
\newblock In \emph{Proceedings of the International Conference on Learning Representations (ICLR)}.

\bibitem[{Henning et~al.(2023)Henning, Beluch, Fraser, and Friedrich}]{henning2023survey}
Sophie Henning, William Beluch, Alexander Fraser, and Annemarie Friedrich. 2023.
\newblock A survey of methods for addressing class imbalance in deep-learning based natural language processing.
\newblock In \emph{Proceedings of the Conference of the European Chapter of the Association for Computational Linguistics (EACL)}, pages 523--540.

\bibitem[{Huang et~al.(2019)Huang, Cheng, Bapna, Firat, Chen, Chen, Lee, Ngiam, Le, Wu, and Chen}]{huang2019gpipe}
Yanping Huang, Youlong Cheng, Ankur Bapna, Orhan Firat, Dehao Chen, Mia Chen, HyoukJoong Lee, Jiquan Ngiam, Quoc~V Le, Yonghui Wu, and zhifeng Chen. 2019.
\newblock Gpipe: Efficient training of giant neural networks using pipeline parallelism.
\newblock \emph{Proceedings of the Conference on Neural Information Processing Systems (NeurIPS)}, 32.

\bibitem[{Jin et~al.(2021)Jin, Pan, Oufattole, Weng, Fang, and Szolovits}]{jin2021disease}
Di~Jin, Eileen Pan, Nassim Oufattole, Wei-Hung Weng, Hanyi Fang, and Peter Szolovits. 2021.
\newblock What disease does this patient have? a large-scale open domain question answering dataset from medical exams.
\newblock \emph{Applied Sciences}, 11(14):6421.

\bibitem[{Jin et~al.(2019)Jin, Dhingra, Liu, Cohen, and Lu}]{jin2019pubmedqa}
Qiao Jin, Bhuwan Dhingra, Zhengping Liu, William Cohen, and Xinghua Lu. 2019.
\newblock Pubmedqa: A dataset for biomedical research question answering.
\newblock In \emph{Proceedings of the Conference on Empirical Methods in Natural Language Processing and the 9th International Joint Conference on Natural Language Processing (EMNLP-IJCNLP)}, pages 2567--2577.

\bibitem[{Katharopoulos and Fleuret(2018)}]{katharopoulos2018not}
Angelos Katharopoulos and François Fleuret. 2018.
\newblock \href {https://arxiv.org/abs/1803.00942} {Not all samples are created equal: Deep learning with importance sampling}.
\newblock In \emph{Proceedings of the International Conference on Machine Learning (ICML)}, pages 2525--2534.

\bibitem[{Li et~al.(2024{\natexlab{a}})Li, Zhang, He, Li, Zhao, Wang, Cheng, and Zhou}]{li2024superfiltering}
Ming Li, Yong Zhang, Shwai He, Zhitao Li, Hongyu Zhao, Jianzong Wang, Ning Cheng, and Tianyi Zhou. 2024{\natexlab{a}}.
\newblock Superfiltering: Weak-to-strong data filtering for fast instruction-tuning.
\newblock In \emph{Proceedings of the Annual Meeting of the Association for Computational Linguistics (ACL)}, pages 14255--14273.

\bibitem[{Li et~al.(2024{\natexlab{b}})Li, Zhang, Li, Chen, Chen, Cheng, Wang, Zhou, and Xiao}]{li2024quantity}
Ming Li, Yong Zhang, Zhitao Li, Jiuhai Chen, Lichang Chen, Ning Cheng, Jianzong Wang, Tianyi Zhou, and Jing Xiao. 2024{\natexlab{b}}.
\newblock From quantity to quality: Boosting llm performance with self-guided data selection for instruction tuning.
\newblock In \emph{Proceedings of the North American Chapter of the Association for Computational Linguistics (NAACL)}, pages 7602--7635.

\bibitem[{Li et~al.(2014)Li, Andersen, Park, Smola, Ahmed, Josifovski, Long, Shekita, and Su}]{li2014parameter}
Mu~Li, David~G Andersen, Jun~Woo Park, Alexander~J Smola, Amr Ahmed, Vanja Josifovski, James Long, Eugene~J Shekita, and Bor-Yiing Su. 2014.
\newblock Scaling distributed machine learning with the parameter server.
\newblock In \emph{Proceedings of the USENIX Symposium on Operating Systems Design and Implementation (OSDI)}, pages 583--598. USENIX Association.

\bibitem[{Liu et~al.(2023)Liu, Xia, Wang, and Zhang}]{liu2023your}
Jiawei Liu, Chunqiu~Steven Xia, Yuyao Wang, and Lingming Zhang. 2023.
\newblock Is your code generated by chatgpt really correct? rigorous evaluation of large language models for code generation.
\newblock \emph{Proceedings of the Conference on Neural Information Processing Systems (NeurIPS)}, 36:21558--21572.

\bibitem[{Liu et~al.(2024)Liu, Zeng, He, Jiang, and He}]{liu2023makes}
Wei Liu, Weihao Zeng, Keqing He, Yong Jiang, and Junxian He. 2024.
\newblock \href {https://openreview.net/forum?id=BTKAeLqLMw} {What makes good data for alignment? a comprehensive study of automatic data selection in instruction tuning}.
\newblock In \emph{The Twelfth International Conference on Learning Representations}.

\bibitem[{Lloyd(1982)}]{lloyd1982least}
Stuart Lloyd. 1982.
\newblock Least squares quantization in pcm.
\newblock \emph{IEEE Transactions on Information Theory}, 28(2):129--137.

\bibitem[{Loshchilov and Hutter(2016)}]{loshchilov2016online}
Ilya Loshchilov and Frank Hutter. 2016.
\newblock \href {https://arxiv.org/abs/1511.06343} {Online batch selection for faster training of neural networks}.
\newblock In \emph{ICLR Workshop}.

\bibitem[{Luo et~al.(2024)Luo, Xu, Zhao, Sun, Geng, Hu, Tao, Ma, Lin, and Jiang}]{luo2024wizardcoder}
Ziyang Luo, Can Xu, Pu~Zhao, Qingfeng Sun, Xiubo Geng, Wenxiang Hu, Chongyang Tao, Jing Ma, Qingwei Lin, and Daxin Jiang. 2024.
\newblock \href {https://openreview.net/forum?id=UnUwSIgK5W} {Wizardcoder: Empowering code large language models with evol-instruct}.
\newblock In \emph{Proceedings of the International Conference on Learning Representations (ICLR)}.

\bibitem[{Mindermann et~al.(2022)Mindermann, Brauner, Razzak, Sharma, Kirsch, Xu, Höltgen, Gomez, Morisot, Farquhar, and Gal}]{mindermann2022prioritized}
Sören Mindermann, Jan Brauner, Muhammed Razzak, Mrinank Sharma, Andreas Kirsch, Winnie Xu, Benedikt Höltgen, Aidan~N. Gomez, Adrien Morisot, Sebastian Farquhar, and Yarin Gal. 2022.
\newblock Prioritized training on points that are learnable, worth learning, and not yet learnt.
\newblock In \emph{Proceedings of the International Conference on Machine Learning (ICML)}, pages 15630--15649. PMLR.

\bibitem[{Pal et~al.(2022)Pal, Umapathi, and Sankarasubbu}]{pal2022medmcqa}
Ankit Pal, Logesh~Kumar Umapathi, and Malaikannan Sankarasubbu. 2022.
\newblock Medmcqa: A large-scale multi-subject multi-choice dataset for medical domain question answering.
\newblock In \emph{Proceedings of the Conference on Health, Inference, and Learning}, pages 248--260. PMLR.

\bibitem[{Rajbhandari et~al.(2020)Rajbhandari, Rasley, Ruwase, and He}]{rajbhandari2020zero}
Samyam Rajbhandari, Jeff Rasley, Olatunji Ruwase, and Yuxiong He. 2020.
\newblock Zero: Memory optimizations toward training trillion parameter models.
\newblock In \emph{Proceedings of the International Conference for High Performance Computing, Networking, Storage and Analysis (SC)}, pages 1--16. IEEE.

\bibitem[{Rozière et~al.(2023)Rozière, Gehring, Gloeckle, Sootla, Gat, Tan, Adi, Liu, Sauvestre, Remez, Rapin, Kozhevnikov, Evtimov, Bitton, Bhatt, Ferrer, Grattafiori, Xiong, Défossez, Copet, Azhar, Touvron, Martin, Usunier, Scialom, and Synnaeve}]{roziere2023code}
Baptiste Rozière, Jonas Gehring, Fabian Gloeckle, Sten Sootla, Itai Gat, Xiaoqing~Ellen Tan, Yossi Adi, Jingyu Liu, Romain Sauvestre, Tal Remez, Jérémy Rapin, Artyom Kozhevnikov, Ivan Evtimov, Joanna Bitton, Manish Bhatt, Cristian~Canton Ferrer, Aaron Grattafiori, Wenhan Xiong, Alexandre Défossez, and 7 others. 2023.
\newblock Code llama: Open foundation models for code.
\newblock \emph{arXiv preprint arXiv:2308.12950}.

\bibitem[{Sener and Savarese(2018)}]{sener2018active}
Ozan Sener and Silvio Savarese. 2018.
\newblock \href {https://arxiv.org/abs/1708.00489} {Active learning for convolutional neural networks: A core-set approach}.
\newblock In \emph{Proceedings of the International Conference on Learning Representations (ICLR)}.

\bibitem[{Shazeer and Stern(2018)}]{shazeer2018adafactor}
Noam Shazeer and Mitchell Stern. 2018.
\newblock Adafactor: Adaptive learning rates with sublinear memory cost.
\newblock In \emph{Proceedings of the International Conference on Machine Learning (ICML)}, pages 4596--4604. PMLR.

\bibitem[{Shoeybi et~al.(2019)Shoeybi, Patwary, Puri, LeGresley, Casper, and Catanzaro}]{shoeybi2019megatron}
Mohammad Shoeybi, Mostofa Patwary, Raul Puri, Patrick LeGresley, Jared Casper, and Bryan Catanzaro. 2019.
\newblock Megatron-lm: Training multi-billion parameter language models using model parallelism.
\newblock \emph{arXiv preprint arXiv:1909.08053}.

\bibitem[{Touvron et~al.(2023)Touvron, Martin, Stone, Albert, Almahairi, Babaei, Bashlykov, Batra, Bhargava, Bhosale, Bikel, Blecher, Ferrer, Chen, Cucurull, Esiobu, Fernandes, Fu, Fu, Fuller, Gao, Goswami, Goyal, Hartshorn, Hosseini, Hou, Inan, Kardas, Kerkez, Khabsa, Kloumann, Korenev, Koura, Lachaux, Lavril, Lee, Liskovich, Lu, Mao, Martinet, Mihaylov, Mishra, Molybog, Nie, Poulton, Reizenstein, Rungta, Saladi, Schelten, Silva, Smith, Subramanian, Tan, Tang, Taylor, Williams, Kuan, Xu, Yan, Zarov, Zhang, Fan, Kambadur, Narang, Rodriguez, Stojnic, Edunov, and Scialom}]{touvron2023llama}
Hugo Touvron, Louis Martin, Kevin Stone, Peter Albert, Amjad Almahairi, Yasmine Babaei, Nikolay Bashlykov, Soumya Batra, Prajjwal Bhargava, Shruti Bhosale, Dan Bikel, Lukas Blecher, Cristian~Canton Ferrer, Moya Chen, Guillem Cucurull, David Esiobu, Jude Fernandes, Jeremy Fu, Wenyin Fu, and 49 others. 2023.
\newblock Llama 2: Open foundation and fine-tuned chat models.
\newblock \emph{arXiv preprint arXiv:2307.09288}.

\bibitem[{Vaswani et~al.(2017)Vaswani, Shazeer, Parmar, Uszkoreit, Jones, Gomez, Kaiser, and Polosukhin}]{vaswani2017attention}
Ashish Vaswani, Noam Shazeer, Niki Parmar, Jakob Uszkoreit, Llion Jones, Aidan~N Gomez, {\L}ukasz Kaiser, and Illia Polosukhin. 2017.
\newblock Attention is all you need.
\newblock \emph{Proceedings of the Conference on Neural Information Processing Systems (NeurIPS)}, 30.

\bibitem[{Wei et~al.(2022)Wei, Bosma, Zhao, Guu, Yu, Lester, Du, Dai, and Le}]{wei2021finetuned}
Jason Wei, Maarten Bosma, Vincent Zhao, Kelvin Guu, Adams~Wei Yu, Brian Lester, Nan Du, Andrew~M. Dai, and Quoc~V Le. 2022.
\newblock \href {https://openreview.net/forum?id=gEZrGCozdqR} {Finetuned language models are zero-shot learners}.
\newblock In \emph{Proceedings of the International Conference on Learning Representations (ICLR)}.

\bibitem[{Wei et~al.(2024)Wei, Wang, Liu, Ding, and Zhang}]{wei2024magicoder}
Yuxiang Wei, Zhe Wang, Jiawei Liu, Yifeng Ding, and Lingming Zhang. 2024.
\newblock Magicoder: empowering code generation with oss-instruct.
\newblock In \emph{Proceedings of the International Conference on Machine Learning (ICML)}, pages 52632--52657.

\bibitem[{Xia et~al.(2024)Xia, Malladi, Gururangan, Arora, and Chen}]{xia2024less}
Mengzhou Xia, Sadhika Malladi, Suchin Gururangan, Sanjeev Arora, and Danqi Chen. 2024.
\newblock Less: selecting influential data for targeted instruction tuning.
\newblock In \emph{Proceedings of the International Conference on Machine Learning (ICML)}, pages 54104--54132.

\bibitem[{Yang et~al.(2024)Yang, Mishra, Chiang, and Mirzasoleiman}]{yang2024smalltolarge}
Yu~Yang, Siddhartha Mishra, Jeffrey Chiang, and Baharan Mirzasoleiman. 2024.
\newblock Smalltolarge (s2l): Scalable data selection for fine-tuning large language models by summarizing training trajectories of small models.
\newblock \emph{Proceedings of the Conference on Neural Information Processing Systems (NeurIPS)}, 37:83465--83496.

\bibitem[{Yin et~al.(2018)Yin, Pananjady, Lam, Papailiopoulos, Ramchandran, and Bartlett}]{yin2018gradient}
Dong Yin, Ashwin Pananjady, Max Lam, Dimitris Papailiopoulos, Kannan Ramchandran, and Peter Bartlett. 2018.
\newblock Gradient diversity: a key ingredient for scalable distributed learning.
\newblock In \emph{Proceedings of the International Conference on Artificial Intelligence and Statistics (AISTATS)}, pages 1998--2007. PMLR.

\bibitem[{Zhang et~al.(2023)Zhang, Tian, Yang, Chen, Li, and Petzold}]{zhang2023alpacare}
Xinlu Zhang, Chenxin Tian, Xianjun Yang, Lichang Chen, Zekun Li, and Linda~Ruth Petzold. 2023.
\newblock Alpacare: Instruction-tuned large language models for medical application.
\newblock \emph{arXiv preprint arXiv:2310.14558}.

\bibitem[{Zhao et~al.(2023)Zhao, Gu, Varma, Luo, Huang, Xu, Wright, Shojanazeri, Ott, Shleifer, Desmaison, Balioglu, Damania, Nguyen, Chauhan, Hao, Mathews, and Li}]{zhao2023pytorch}
Yanli Zhao, Andrew Gu, Rohan Varma, Liang Luo, Chien-Chin Huang, Min Xu, Less Wright, Hamid Shojanazeri, Myle Ott, Sam Shleifer, Alban Desmaison, Can Balioglu, Pritam Damania, Bernard Nguyen, Geeta Chauhan, Yuchen Hao, Ajit Mathews, and Shen Li. 2023.
\newblock \href {https://doi.org/10.14778/3611540.3611569} {Pytorch fsdp: Experiences on scaling fully sharded data parallel}.
\newblock \emph{Proc. VLDB Endow.}, 16(12):3848–3860.

\end{thebibliography}

\appendix

\appendix

\clearpage
\section{Proofs for Section~\ref{sec:sampling-variance}}
\label{app:3_proofs}

\subsection{Notations}
\label{app:notation}

For clarity, we summarize the main notations used in the theoretical analysis. 
We assume that the dataset consists of $N$ samples partitioned into $K$ clusters of equal size $n$. 
Each sample $i$ is associated with a gradient vector $g_i \in \mathbb{R}^d$. 
The global mean gradient is denoted by $\mu$, while $\mu_c$ represents the mean gradient of cluster $c$. 
The intra-cluster variance $S_c^2$ measures the variability of gradients within cluster $\mathcal{C}_c$. 
Table~\ref{tab:notation} summarizes these symbols and their descriptions.

\subsection{Intra- and Inter-Cluster Decomposition}
\label{app:prop1}
Let $\mu=\tfrac{1}{N}\sum_{i=1}^N g_i$ and $\mu_c=\tfrac{1}{N_c}\sum_{i\in \mathcal{C}_c} g_i$.
Then the total sum of squares splits into an \emph{intra-cluster} part and an
\emph{inter-cluster} part:
% \[
% \sum_{i=1}^N (g_i-\mu)^2
% =\sum_{k=1}^K \sum_{i\in k} (g_{k,i}-\mu_k)^2
% +\sum_{k=1}^K N_k(\mu_k-\mu)^2.
% \]
% \[
% \begin{aligned}
% \sum_{i=1}^N (g_i-\mu)^2
% &= \sum_{c=1}^K \sum_{i\in \mathcal{C}_c} (g_{c,i}-\mu_c)^2 \\
% &\quad + \sum_{c=1}^K N_c(\mu_c-\mu)^2 .
% \end{aligned}
% \]
\[
\begin{aligned}
\sum_{i=1}^N \left\| g_i-\mu \right\|_2^2
&=
\sum_{c=1}^K \sum_{i\in \mathcal{C}_c}
\left\| g_{c,i}-\mu_c \right\|_2^2 \\
&\quad +
\sum_{c=1}^K N_c
\left\| \mu_c-\mu \right\|_2^2 .
\end{aligned}
\]
Dividing by $N$ gives $S^2=S_{\text{intra}}^2+S_{\text{inter}}^2$; for equal $N_c$,
$S_{\text{intra}}^2=\tfrac{1}{K}\sum_c S_c^2$ and $S_{\text{inter}}^2=\tfrac{1}{K}\sum_{c=1}^{K}\|\mu_c-\mu\|_2^2$.
% $S_{\text{inter}}^2=\tfrac{1}{K}\sum_c(\mu_c-\mu)^2$.

\subsection{Proof of Proposition~\ref{prop:srs}}
\label{app:proofsrs}
Under simple random sampling without replacement, 
we uniformly select $B$ distinct samples from the $N$ total data points ($B<N$). 
Let $I_i \in \{0,1\}$ be an indicator variable denoting whether the $i$-th sample is selected. 
Then we have
\[
\begin{gathered}
\mathbb{E}[I_i] = \tfrac{B}{N}, \\
\mathrm{Var}(I_i) = \tfrac{B}{N}\Big(1-\tfrac{B}{N}\Big), \\
\mathrm{Cov}(I_i, I_j) = -\tfrac{B(N-B)}{N^2(N-1)}.
\end{gathered}
\]
Using $\hat G = \tfrac{1}{B}\sum_{i=1}^N I_i g_i$, we have
% \begin{align}
% \mathrm{Var}(\hat G)
% &= \tfrac{1}{B^2}\,\mathrm{Var}\!\Big(\sum_{i=1}^N I_i g_i\Big)
% = \tfrac{N-B}{BN}\,\tfrac{N}{N-1}\,S^2 \\
% \approx \tfrac{1-f}{B}\,S^2 \quad (\frac{N}{N-1} \approx 1 \text{ for large } N).
% \end{align}
\begin{align}
\mathrm{Var}(\hat G)
&= \tfrac{1}{B^2}\,\mathrm{Var}\!\Big(\sum_{i=1}^N I_i g_i\Big) \\
&= \tfrac{N-B}{BN}\,\tfrac{N}{N-1}\,S^2 \\
&\approx \tfrac{1-f}{B}\,S^2 \left(\tfrac{N}{N-1} \approx 1 \text{ for large } N\right).
\end{align}
Substituting the variance decomposition $S^2 = S_{\text{intra}}^2 + S_{\text{inter}}^2$ yields
\begin{equation}
\label{eq:var-random}
\mathrm{Var}(\hat G_{\mathrm{Random}}) 
= \tfrac{1-f}{B}\big(S_{\text{intra}}^2 + S_{\text{inter}}^2\big).
\end{equation}

\begin{table}[t]
\caption{Summary of notation used in the theoretical analysis. }
\label{tab:notation}
\resizebox{\linewidth}{!}{
\begin{tabular}{ll}
\toprule
Symbol & Description \\
\midrule
$N$ & Total number of samples in the dataset \\
$K$ & Number of clusters \\
$N_c$ & Number of samples in cluster $\mathcal{C}_c$ ($N = Kn$) \\
$n$ & Number of samples in balanced cluster ($N = Kn$) \\
$\mathcal{C}_c$ & Set of samples belonging to $c$-th cluster \\
$g_i \in \mathbb{R}^d$ & Gradient of sample $i$ \\
$\mu$ & Population mean gradient $\frac{1}{N}\sum_{i=1}^{N} g_i$ \\
$\mu_c$ & Mean gradient of cluster $c$ \\
$S_c^2$ & Intra-cluster gradient variance \\
\bottomrule
\end{tabular}
}
\end{table}

\paragraph{Derivation.}
Using $\hat G = \frac{1}{B}\sum_{i=1}^N I_i g_i$, where 
$\mathbb{E}[I_i] = \frac{B}{N}$,
$\mathrm{Var}(I_i) = \frac{B}{N}(1-\frac{B}{N})$, and
$\mathrm{Cov}(I_i, I_j) = -\frac{B(N-B)}{N^2(N-1)}$ for $i\neq j$, we have
% \begin{aligned}
% \mathrm{Var}(\hat G)
% &= \frac{1}{B^2}\mathrm{Var}\Big(\sum_{i=1}^N I_i g_i\Big) \\
% &= \frac{1}{B^2}\Big[
%    \sum_{i=1}^N g_i^2 \mathrm{Var}(I_i)
%    + 2\sum_{i<j} g_i g_j \mathrm{Cov}(I_i, I_j)
%    \Big].
% \end{aligned}
\begin{equation}
\begin{aligned}
\mathrm{Var}(\hat G)
&= \frac{1}{B^2}\mathrm{Var}\!\left(\sum_{i=1}^N I_i g_i\right) \\
&= \frac{1}{B^2}\Bigg(
   \sum_{i=1}^N g_i^2 \mathrm{Var}(I_i) \\
&\qquad + 2\sum_{i<j} g_i g_j \mathrm{Cov}(I_i, I_j)
   \Bigg).
\end{aligned}
\end{equation}
Let $\mu = \frac{1}{N}\sum_{i=1}^N g_i$ and 
$S^2 = \frac{1}{N}\sum_{i=1}^N (g_i - \mu)^2$. 
Using the identity $\sum_{i<j} g_i g_j 
= \frac{1}{2}\big((\sum_{i} g_i)^2 - \sum_{i} g_i^2\big)$,
we simplify the expression to
\begin{align*}
\mathrm{Var}(\hat G)
&= \frac{N-B}{BN}\frac{N}{N-1} S^2
= \frac{1-f}{B}\frac{N}{N-1} S^2,
\end{align*}
where $f = \frac{B}{N}$ is the sampling fraction.
For a large population ($N \gg 1$),
$\frac{N}{N-1} \approx 1$, yielding ~\eqref{eq:var-random}.

\subsection{Proof of Proposition~\ref{prop:strat}}
\label{app:proofstrat}
Consider uniform sampling, in which we draw $b = B/K$ samples per cluster and compute the weighted average of gradients. 
Let $W_c = N_c / N = 1/K$. 
Then the gradient estimator is
\[
\hat G_{\mathrm{uniform}} = \sum_{c=1}^K W_c\,\bar g_c,
\qquad 
\bar g_c = \frac{1}{b} \sum_{i \in \mathcal{C}_c} g_i.
\]
Since sampling across clusters is independent,
\[
\mathrm{Var}\!\Big(\sum_{c=1}^K W_c\,\bar g_c\Big)
= \sum_{c=1}^K W_c^2\,\mathrm{Var}(\bar g_c).
\]
Within each cluster, sampling $b$ items without replacement gives 
$\mathrm{Var}(\bar g_c) = (1-f_c)\,S_c^2/b$ with $f_c = b/N_c = f$. 
Substituting $W_c = 1/K$ and $b = B/K$ yields
\begin{align}
\mathrm{Var}(\hat G_{\text{uniform}})
&= \sum_{c=1}^K \frac{W_c^2}{b}\,(1-f_c)\,S_c^2 \nonumber\\[-4pt]
&= \frac{1-f}{B}\cdot \frac{1}{K}\sum_{c=1}^K S_c^2
= \frac{1-f}{B}\,S_{\text{intra}}^2.
\label{eq:var-strat}
\end{align}

\subsection{Proof of Proposition~\ref{prop:comparison}}
\label{app:prop3}
From Proposition~\ref{prop:srs}, the variance of the random sampling estimator is
\[
\mathrm{Var}(\hat G_{\mathrm{Random}})
=
\frac{1-f}{B}\left(S_{\text{intra}}^2 + S_{\text{inter}}^2\right).
\]
From Proposition~\ref{prop:strat}, the variance of the uniform sampling estimator is
\[
\mathrm{Var}(\hat G_{\mathrm{strat}})
=
\frac{1-f}{B}S_{\text{intra}}^2.
\]
Taking the difference between the two expressions yields
\[
\mathrm{Var}(\hat G_{\mathrm{Random}})
-
\mathrm{Var}(\hat G_{\mathrm{strat}})
=
\frac{1-f}{B}S_{\text{inter}}^2.
\]
Since \(S_{\text{inter}}^2 \ge 0\), it follows that
\[
\mathrm{Var}(\hat G_{\mathrm{strat}})
\le
\mathrm{Var}(\hat G_{\mathrm{Random}}).
\]
Moreover, the inequality is strict whenever \(S_{\text{inter}}^2 > 0\), and equality holds if and only if \(S_{\text{inter}}^2 = 0\).

However, uniform stratification requires pre-balanced clusters and enforces equal sampling weights, which distorts the original data distribution and may increase sampling redundancy.
Although theoretically optimal under a balanced partition, it is difficult to apply in real-world heterogeneous datasets.

\section{Gradient Diversity Analysis under Distance-Based Selection}
\label{appendix:prune_strategy}

To understand the effect of within-cluster selection on optimization, we analyze 
how distance-based selection influences the gradient diversity of each 
mini-batch. 
Following \citep{yin2018gradient}, the gradient diversity of a mini-batch 
$\mathcal{S}$ is defined as
\begin{equation}
    \Delta(\mathcal{S})
    =
    \frac{\sum_{i\in\mathcal{S}} \| g_i \|^2}
         {\left\| \sum_{i\in\mathcal{S}} g_i \right\|^2},
    \label{eq:grad-div}
\end{equation}
which measures directional variability among gradients. 
Expanding the denominator yields
\begin{equation}
    \left\| \sum_{i\in\mathcal{S}} g_i \right\|^2
    =
    \sum_{i}\|g_i\|^2
    \;+\;
    2\sum_{i<j} \langle g_i, g_j \rangle,
\end{equation}
so that $\Delta(\mathcal{S})$ decreases when the pairwise inner products 
$\langle g_i, g_j\rangle$ increase.  
Thus, batches with more coherent gradient directions exhibit lower diversity.
Equivalently, because
\[
    \langle g_i, g_j\rangle 
    = 
    \|g_i\|\,\|g_j\|\cos\theta_{ij},
\]
smaller cosine values (i.e., larger mutual angles) reduce the denominator of 
$\Delta(\mathcal{S})$ and therefore \emph{increase} the gradient diversity.

\paragraph{Gradient Decomposition.}
Within a cluster $c$, we model each gradient as
\begin{equation}
g_i = \bar{\mu}_c + \varepsilon_i,
\end{equation}
where $\bar{\mu}_c \in \mathbb{R}^d$ is the cluster mean
and $\varepsilon_i \sim \mathcal{N}(0,\sigma^2 I)$
denotes isotropic residual noise.
We assume $\|\varepsilon_i\| \ll \|\bar{\mu}_c\|$.

\paragraph{Norm Expansion.}
Using a first-order Taylor expansion,
\begin{align}
\|g_i\|
&=
\|\bar{\mu}_c + \varepsilon_i\| \\
&=
\|\bar{\mu}_c\|
\left(
1 +
\frac{\bar{\mu}_c^\top \varepsilon_i}{\|\bar{\mu}_c\|^2}
\right)
+ O(\|\varepsilon_i\|^2).
\end{align}
Taking the reciprocal,
\begin{equation}
\frac{1}{\|g_i\|}
\approx
\frac{1}{\|\bar{\mu}_c\|}
\left(
1 -
\frac{\bar{\mu}_c^\top \varepsilon_i}{\|\bar{\mu}_c\|^2}
\right).
\end{equation}

\paragraph{Normalized Gradient Approximation.}
The normalized gradient $u_i = g_i / \|g_i\|$ becomes
\begin{equation}
u_i
\approx
\frac{\bar{\mu}_c}{\|\bar{\mu}_c\|}
+
\frac{1}{\|\bar{\mu}_c\|}
\left(
\varepsilon_i
-
\frac{\bar{\mu}_c^\top \varepsilon_i}{\|\bar{\mu}_c\|^2}
\bar{\mu}_c
\right).
\end{equation}
Define $v = \bar{\mu}_c / \|\bar{\mu}_c\|$.
Then $u_i \approx v + \delta_i$, where
\[
\delta_i
=
\frac{1}{\|\bar{\mu}_c\|}
\left(
\varepsilon_i
-
(\varepsilon_i^\top v) v
\right).
\]
Note that $\delta_i$ lies in the $(d-1)$-dimensional subspace orthogonal to $v$.

\paragraph{Inner Product Expectation.}
We now compute
\[
\mathbb{E}[\langle u_i, u_j \rangle]
=
\mathbb{E}[\langle v + \delta_i,\, v + \delta_j \rangle].
\]
Expanding,
\[
\langle u_i, u_j \rangle
=
1
+
\langle v, \delta_j \rangle
+
\langle \delta_i, v \rangle
+
\langle \delta_i, \delta_j \rangle.
\]
Since $\mathbb{E}[\delta_i]=0$ and
$\mathbb{E}[\delta_i^\top v]=0$,
the linear terms vanish in expectation.
Thus,
\[
\mathbb{E}[\langle u_i, u_j \rangle]
=
1
+
\mathbb{E}[\langle \delta_i, \delta_j \rangle].
\]
Using independence and isotropy,
\[
\mathbb{E}[\langle \delta_i, \delta_j \rangle]
=
- \frac{(d-1)\sigma^2}{\|\bar{\mu}_c\|^2}.
\]
Therefore,
\begin{equation}
\mathbb{E}[\langle u_i, u_j \rangle]
\approx
1
-
\frac{(d-1)\sigma^2}{\|\bar{\mu}_c\|^2}
+ O(\sigma^4),
\end{equation}
which completes the derivation.

\section{Details of the Weighted Gradient Update}
\label{appendix:weight-derivation}
\begin{table}[t]
\centering
\caption{Centroid similarity between full clusters and selected subsets obtained from Step 2 (data selection) of our method for \textrm{CodeLlama-7B} fine-tuned on \textrm{Magicoder-OSS-Instruct-75K}. The selection ratio is $1.00$.}
\label{tab:centroid_preservation}
\resizebox{\columnwidth}{!}{
\begin{NiceTabular}{c c c c}
\toprule
Cluster & $|\mathcal{C}_c|$ & Cosine Similarity & $\ell_2$ Distance \\
\midrule
0 & 18216 & 0.9997 & 0.33 \\
1 & 21283 & 0.9993 & 0.46 \\
2 & 14901 & 1.0000 & 0.00 \\
3 & 20797 & 0.9996 & 0.29 \\
\bottomrule
\end{NiceTabular}
}
\end{table}
\begin{algorithm*}[t]
\caption{Algorithm of \algname{}}
\label{alg:casdp}
    \begin{algorithmic}[1]
    \Require Dataset $\mathcal{D}$, number of GPUs $K$, selection ratio $r$, cluster capacity $\alpha$, mini-batch size $B$, number of training steps $T$
    \State Compute the proposed gradient-proxy embeddings $\{ g(x) \mid x \in \mathcal{D} \}$ using BADGE
    % \State Run $K$-means on $\{ g(x) \}$ to obtain clusters $\{ \mathcal{C}_1, \ldots, \mathcal{C}_K \}$
    \State Run capacity-constrained $K$-means on $\{ g(x) \}$ with $|\mathcal{C}_c| \le \frac{\alpha|\mathcal{D}|}{K}$ to obtain clusters $\{ \mathcal{C}_1, \ldots, \mathcal{C}_K \}$
    \State Let $n_c = |\mathcal{C}_c|$ for $c = 1,\ldots,K$ and $N = \sum_{c=1}^K n_c$
    \State Compute the smallest cluster size $n_{\min} = \min_{c} n_c$
    \For{$c = 1, \ldots, K$} \Comment{balanced peripheral selection}
        \State Compute cluster centroid $\mu_c = \frac{1}{n_c} \sum_{x \in \mathcal{C}_c} g(x)$
        \State For all $x \in \mathcal{C}_c$, compute distance $d(x) = \| g(x) - \mu_c \|_2$
        \State Sort $\mathcal{C}_c$ in descending order of $d(x)$ \Comment{farthest samples first}
        \State Define balanced subset $\tilde{\mathcal{C}}_c$ as the top $n_{\min}$ samples in the sorted list
        \State Define pruned subset $\hat{\mathcal{C}}_c$ as the top $\lfloor r n_{\min} \rfloor$ samples in $\tilde{\mathcal{C}}_c$
        \State Assign $\hat{\mathcal{C}}_c$ to GPU $c$
        \State Set cluster weight $w_c = \frac{n_c K}{N}$
    \EndFor
    \For{$t = 1, \ldots, T$} \Comment{data-parallel training}
        \For{$c = 1, \ldots, K$ \textbf{in parallel}}
            \State Sample a local mini-batch $\mathcal{P}_c \subseteq \hat{\mathcal{C}}_c$ with $|\mathcal{P}_c| = B / K$
            \State Compute local gradient 
            \[
                g_c = \frac{1}{|\mathcal{P}_c|} \sum_{x \in \mathcal{P}_c} \nabla_\theta \ell_x(\theta)
            \]
            \State Scale local gradient by its cluster weight: $\tilde{g}_c = w_c \, g_c$
        \EndFor
        \State Aggregate $\tilde{g} = \frac{1}{K} \sum_{c=1}^K \tilde{g}_c$ via All-Reduce
        \State Update model parameters: $\theta \leftarrow \theta - \eta \tilde{g}$
    \EndFor
    \end{algorithmic}
    \vspace*{-0.1cm}
\end{algorithm*}
\begin{table*}[t]
\centering
\caption{Statistics of the training datasets and evaluation benchmarks.}
\label{tab:data_statistics}
\resizebox{\textwidth}{!}{
\begin{tabular}{lll}
\toprule
\textbf{Artifact} & \textbf{Domain / Task} & \textbf{Statistics} \\
\midrule
\multicolumn{3}{l}{\textbf{Training datasets}} \\
Magicoder-OSS-Instruct-75K & Code instruction tuning & 75,197 training examples \\
Evol-Instruct-Code-80K & Code instruction tuning & 78,264 training examples \\
MedInstruct-52K & Medical instruction tuning & 52,002 training examples \\
\midrule
\multicolumn{3}{l}{\textbf{Evaluation benchmarks}} \\
HumanEval & Code generation & 164 test problems \\
HumanEval+ & Code generation & 164 test problems with extended test cases \\
MBPP & Code generation & 974 programming problems \\
MBPP+ & Code generation & MBPP-based evaluation with enhanced test cases \\
MedMCQA & Medical QA & 4,183 validation questions \\
MedQA & Medical QA & 1,273 test questions \\
PubMedQA & Biomedical QA & 1,000 expert-labeled test questions \\
MMLU medical subsets & Medical QA & 
\makecell[tl]{945 test questions across 5 medical subsets: \\ anatomy, clinical knowledge, college medicine,\\
medical genetics, professional medicine} \\
\bottomrule
\end{tabular}
}
\end{table*}

\subsection{Derivation of the Weighting Coefficient}

We compare the gradient estimation under random sampling and our weighted sampling scheme. 
Let the dataset be partitioned into $K$ clusters $\{\mathcal{C}_1,\dots,\mathcal{C}_K\}$ with sizes $n_c = |\mathcal{C}_c|$ and total size $N = \sum_{c=1}^K n_c$.
We denote by $\ell_i(\theta)$ the loss of sample $i$ and by $\nabla_\theta \ell_i(\theta)$ its gradient.

\paragraph{Random sampling with full dataset.}
Under standard random sampling from the full dataset, 
we construct a mini-batch $\mathcal{B} \subset \mathcal{D}$ of size $|\mathcal{B}| = B$ and estimate the gradient as
\begin{equation}
    g_{\mathrm{rand}}
    = \frac{1}{B} \sum_{i \in \mathcal{B}} \nabla_\theta \ell_i(\theta).
\end{equation}
Let $\mathcal{B}_c = \mathcal{B} \cap \mathcal{C}_c$ denote the subset of the mini-batch belonging to cluster $\mathcal{C}_c$.
In expectation, the number of samples drawn from cluster $\mathcal{C}_c$ satisfies
\begin{equation}
    \mathbb{E}\bigl[ |\mathcal{B}_c| \bigr]
    = B \, \frac{n_c}{N}.
\end{equation}

\paragraph{Balanced cluster-based sampling.}
In our cluster-based setting, we draw the same number of samples from each cluster.
We denote by $\mathcal{P}_c \subseteq \hat{\mathcal{C}}_c$ the subset sampled from cluster $c$, and set
\begin{equation}
    |\mathcal{P}_c| = \frac{B}{K}, \qquad c = 1,\dots,K.
\end{equation}
We then introduce a cluster-specific weighting coefficient $w_c$ to modulate each cluster’s contribution and define the weighted gradient estimator as
\begin{equation}
    \hat{g}
    = \frac{1}{B} \sum_{c=1}^{K} 
      w_c \sum_{i \in \mathcal{P}_c} \nabla_\theta \ell_i(\theta).
\end{equation}

\paragraph{Deriving the weighting coefficient.}
Assuming that the expected gradient direction within each cluster is preserved across sampling schemes, 
we can compare the relative contribution of cluster $c$ under random and balanced sampling.
Under random sampling, the expected fraction of the mini-batch drawn from cluster $c$ is
\begin{equation}
    \frac{\mathbb{E}[|\mathcal{B}_c|]}{B} = \frac{n_c}{N}.
\end{equation}
Under balanced sampling with weighting, the effective fraction contributed by cluster $c$ is
\begin{equation}
    \frac{w_c |\mathcal{P}_c|}{B} = \frac{w_c}{K}.
\end{equation}
Equating these two fractions yields
\begin{equation}
    \frac{n_c}{N} = \frac{w_c}{K}
    \quad \Rightarrow \quad
    w_c = \frac{n_c K}{N}.
\end{equation}
Thus, the weighting coefficient for cluster $c$ is given by
\begin{equation}
    \label{eq:cluster_weight_clean}
    w_c = \frac{n_c K}{N}.
\end{equation}

Hence, the weighting term compensates for the difference in cluster sizes, preserving the original data distribution while maintaining balanced sampling across clusters.
More detailed empirical validation is in Appendix~\ref{app:empirical_validation}.

\subsection{Empirical Validation of the Centroid Preservation Assumption}
\label{app:empirical_validation}

The derivation above assumes that the expected gradient direction within each cluster is approximately preserved after data selection. 
Formally, we assume that the centroid of gradients computed from the selected subset remains close to the centroid computed from the full cluster.

To validate this assumption empirically, we compute the centroid of gradients for each cluster using the full dataset and compare it with the centroid computed from the pruned subset.
For cluster $c$, we define

\begin{equation}
\begin{aligned}
\mu_c^{\mathrm{all}} = \frac{1}{|\mathcal{C}_c|} \sum_{i \in \mathcal{C}_c} \nabla_\theta \ell_i(\theta), \\
\mu_c^{\mathrm{sel}} = \frac{1}{|\mathcal{P}_c|} \sum_{i \in \mathcal{P}_c} \nabla_\theta \ell_i(\theta).
\end{aligned}
\end{equation}

We measure the cosine similarity and $\ell_2$ distance between $\mu_c^{\mathrm{all}}$ and $\mu_c^{\mathrm{sel}}$ for each cluster.

As shown in Table~\ref{tab:centroid_preservation}, the cosine similarity between the centroids computed from the full clusters and the selected subsets is consistently above 0.999 across all clusters. 
This indicates that the gradient direction within each cluster is largely preserved after pruning.
This behavior can be attributed to the redundancy of samples located near the cluster centroid. 
Since many such samples contribute similar gradients, removing a subset of them has little impact on the overall cluster mean. 
Meanwhile, the selected samples still span the cluster structure, resulting in minimal centroid shift.

These results empirically support the assumption used in the derivation that the expected gradient direction within each cluster remains approximately unchanged across sampling schemes. 
Consequently, the weighting coefficient in Eq.~\eqref{eq:cluster_weight_clean} provides a reasonable approximation of the original gradient contribution from each cluster.
\section{Algorithm of \algname{}}
\label{appendix:alg}

\algname{} is summarized in Algorithm~\ref{alg:casdp}, which describes the cluster-aware balanced sampling and weighted data-parallel training.
\section{Details of Experiment Settings}
\label{app:exp_set_app}
\vspace{-0.2cm}
\subsection{Implementation Details}
\label{app:sub_impl_details}
For our experiments on Magicoder-OSS-Instruct-75K, conducted on four NVIDIA H100 GPUs, we fine-tune the same base model for 2 epochs using PyTorch DDP. In this setting, we adopt Adafactor~\citep{shazeer2018adafactor} with a learning rate of 5e-5, a linear learning rate scheduler, and 15 warm-up steps. %  and a warm-up ratio of 0.1.
The global batch size is set to 512 for all experiments.
For the GPU scalability experiment, we scale the training to eight NVIDIA H100 GPUs while keeping all other training configurations unchanged.

For experiments on the other code domain dataset, we employ CodeLlama-Python-7B~\citep{roziere2023code} as the base model.
For experiments on \textrm{Evol-Instruct-Code-80K}, conducted on four NVIDIA H200 GPUs, we fine-tune the model for 2 epochs using PyTorch FSDP. We use AdamW as optimizer with a learning rate of 5e-5, a linear learning rate scheduler, and a warm-up ratio of 0.05.

Medical domain dataset experiments use \textrm{Llama-2-7B} as the base model.
On \textrm{AlpaCare-MedInstruct-52K}, we fine-tune the model for 3 epochs on four H100 GPUs using PyTorch FSDP.
We adopt AdamW as the optimizer with a learning rate of 2e-5, a cosine learning rate scheduler, and a warm-up ratio of 0.03.
The per-device batch size is set to 4 with 8 gradient accumulation steps, and the maximum input length is 1024.

For the S2L baseline~\citep{yang2024smalltolarge}, we follow its small-to-large setting by using training trajectories from DeepSeek-Coder-1.3B-Instruct, trained for 2 epochs, as the small-model proxy in the code-domain experiments. S2L is evaluated under the same effective data budget as the other data selection baselines.
In the medical-domain experiments, we use \textrm{Llama-3.2-1B-Instruct} as the small model and train for 1 epoch to collect training trajectories. The small model is fine-tuned on the corresponding medical training set, and S2L is applied under the same effective data budget as the other data selection baselines.

In all experiments, we run each configuration with three different random seeds and report the mean and standard deviation of the results. We set the cluster capacity hyperparameter to $\alpha = 1.5$.

\subsection{Dataset Details}
We report dataset statistics in Table~\ref{tab:data_statistics}, including the number of training examples, effective data usage after cluster pruning and sampling, evaluation benchmarks, and repeated-run settings. 
For evaluation, we use the official benchmark splits when available.

% \subsection{Licenses and Terms of Use.}
% We use pretrained models, training datasets, and evaluation benchmarks in accordance with their respective licenses and terms of use. 
% The training datasets are used only for research on data selection and instruction tuning, and the benchmark datasets are used only for evaluation. 
% The pretrained models are used for research-purpose fine-tuning and evaluation under their original access conditions. 
% We do not redistribute the original datasets, benchmark data, or pretrained model weights. 
% Any derived artifacts created in this work, such as selected subsets and gradient-based representations, are intended solely for research, reproduction, and analysis.

\section{Details of Experiment Results}
\label{app:exp_app}

\subsection{Experiment Result on Additional Dataset}

\begin{table}[t]
\centering
\caption{
Pass@1 of sampling strategies for \textrm{CodeLlama-7B} on \textrm{Evol-Instruct-Code-80K}. 
The 68\%, 51\%, and 35\% correspond to pruning ratios, $r = 1.00$, $0.75$, and $0.50$, respectively.
% Pass@1 is evaluated with greedy decoding.
% Pass@1\,(\%) and training time for data sampling strategies on \textit{CodeLlama-7B} fine-tuned with \textit{Evol-Instruct-Code-80K}.
% Cluster balancing keeps 68\% of the original dataset; applying intra-cluster selection ratios $r \in \{1.00, 0.75, 0.50\}$ results in effective data uses of 68\%, 51\%, and 35\% of the original dataset, respectively.
% Pass@1 is evaluated with greedy decoding.
% \textbf{Bold} indicates the best score at each data usage.
% \colorbox{black!10}{Gray} cells indicate reference values for performance and training-time comparison.
} %  and \underline{underline} marks the best overall score across all configurations
\label{tab:main_evol}
\resizebox{\linewidth}{!}{
\begin{tabular}{ll ccc}
\toprule
\multicolumn{2}{c}{Data Selection} & \multicolumn{3}{c}{Evaluation Benchmark} \\
\cmidrule(lr){1-2} \cmidrule(lr){3-5}
\textbf{Ratio} & \textbf{Sampling} & \textbf{HumanEval(+)} & \textbf{MBPP(+)} & \textbf{Average} \\
\midrule
\multirow{1}{*}{\makecell[l]{100\%}}
  & Random  & 48.6{\scriptsize $\pm$ 4.6} (41.9{\scriptsize $\pm$ 5.1}) & 64.5{\scriptsize $\pm$ 1.3} (53.3{\scriptsize $\pm$ 2.9}) & 52.1 \\
  % & Uniform & 43.3{\scriptsize $\pm$ 2.2} (37.4{\scriptsize $\pm$ 2.5}) & 33.0{\scriptsize $\pm$ 3.2} (28.1{\scriptsize $\pm$ 2.8}) & 35.5 & -31.9 & 1.02 & -24.1 \\
\midrule
\multirow{2}{*}{\makecell[l]{68\%}} %No redundancy
  & Random  & 46.7{\scriptsize $\pm$ 1.3} (39.2{\scriptsize $\pm$ 0.9}) & 45.7{\scriptsize $\pm$ 2.1} (38.4{\scriptsize $\pm$ 2.8}) & 42.5 \\
  % & Uniform & 42.9{\scriptsize $\pm$ 2.8} (36.6{\scriptsize $\pm$ 2.8}) & 39.8{\scriptsize $\pm$ 2.4} (33.2{\scriptsize $\pm$ 1.7}) & 38.1 & -26.9 & 0.70 & 14.6 \\
  % & IFD     & 47.0{\scriptsize $\pm$ 1.6} (40.3{\scriptsize $\pm$ 1.6}) & 33.9{\scriptsize $\pm$ 6.0} (27.9{\scriptsize $\pm$ 5.2}) & 37.2 & -28.6 & 0.55 & 32.9 \\
  % & LESS    & {\scriptsize $\pm$ } ({\scriptsize $\pm$ }) & {\scriptsize $\pm$ } ({\scriptsize $\pm$ }) &  &  & 0.56 & 31.2 \\
  & \cellcolor{clustergray}\textbf{\algname{}} 
             & \cellcolor{clustergray}\textbf{50.2{\scriptsize $\pm$ 0.9} (44.5{\scriptsize $\pm$ 0.0})} 
             & \cellcolor{clustergray}\textbf{62.5{\scriptsize $\pm$ 1.6} (52.1{\scriptsize $\pm$ 1.7})} 
             & \cellcolor{clustergray}\textbf{52.3} \\
\midrule
\multirow{2}{*}{\makecell[l]{51\%}} 
  & Random  & 45.3{\scriptsize $\pm$ 3.4} (39.2{\scriptsize $\pm$ 2.8}) & 47.7{\scriptsize $\pm$ 5.7} (39.0{\scriptsize $\pm$ 4.7}) & 42.6 \\
  % & Uniform & 35.4{\scriptsize $\pm$ 5.4} (31.5{\scriptsize $\pm$ 4.9}) & 40.6{\scriptsize $\pm$ 0.9} (33.4{\scriptsize $\pm$ 1.3}) & 35.2 & -32.4 & 0.52 & 36.6 \\
  % & IFD     & 45.1{\scriptsize $\pm$ 1.6} (38.6{\scriptsize $\pm$ 0.9}) & 34.5{\scriptsize $\pm$ 2.2} (28.9{\scriptsize $\pm$ 2.0}) & 36.8 & -29.4 & 0.42 & 48.8 \\
  % & LESS    & {\scriptsize $\pm$ } ({\scriptsize $\pm$ }) & {\scriptsize $\pm$ } ({\scriptsize $\pm$ }) &  &  & 0.44 & 45.8 \\
  & \cellcolor{clustergray}\textbf{\algname{}} 
             & \cellcolor{clustergray}\textbf{48.8{\scriptsize $\pm$ 2.1} (43.5{\scriptsize $\pm$ 1.9})} 
             & \cellcolor{clustergray}\textbf{62.6{\scriptsize $\pm$ 1.0} (52.0{\scriptsize $\pm$ 1.1})} 
             & \cellcolor{clustergray}\textbf{51.7} \\
\midrule
\multirow{2}{*}{\makecell[l]{35\%}} 
  & Random  & 45.6{\scriptsize $\pm$ 3.0} (39.2{\scriptsize $\pm$ 2.3}) & 43.7{\scriptsize $\pm$ 2.3} (36.9{\scriptsize $\pm$ 2.5}) & 41.4 \\
  % & Uniform & 29.7{\scriptsize $\pm$ 5.6} (25.8{\scriptsize $\pm$ 4.1}) & 39.4{\scriptsize $\pm$ 1.9} (32.7{\scriptsize $\pm$ 2.1}) & 31.9 & -38.8 & 0.35 & 57.9 \\
  % & IFD     & 45.5{\scriptsize $\pm$ 1.3} (39.2{\scriptsize $\pm$ 2.0}) & 33.3{\scriptsize $\pm$ 3.0} (27.8{\scriptsize $\pm$ 3.0}) & 36.5 & -30.0 & 0.27 & 67.1 \\
  % & LESS    & {\scriptsize $\pm$ } ({\scriptsize $\pm$ }) & {\scriptsize $\pm$ } ({\scriptsize $\pm$ }) &  &  & 0.28 & 65.7 \\
  & \cellcolor{clustergray}\textbf{\algname{}} 
             & \cellcolor{clustergray}\textbf{47.0{\scriptsize $\pm$ 1.1} (42.1{\scriptsize $\pm$ 1.0})} 
             & \cellcolor{clustergray}\textbf{61.2{\scriptsize $\pm$ 1.3} (50.8{\scriptsize $\pm$ 1.0})} 
             & \cellcolor{clustergray}\textbf{50.3} \\
\bottomrule
\end{tabular}}
\end{table}

Table~\ref{tab:main_evol} reports the Pass@1 results of \textrm{CodeLlama-7B}~\citep{roziere2023code} fine-tuned on \textrm{Evol-Instruct-Code-80K}, comparing our method with random sampling across different data selection ratios.
At every ratio, our method consistently outperforms the random baseline.
Notably, even as the dataset size decreases with higher pruning ratios, the performance degradation remains limited.
In addition, the training time decreases proportionally with the reduced data size, demonstrating that our approach improves training efficiency while maintaining competitive performance.

\begin{table}[t]
\centering
\caption{Ablation study on the effect of the weighting mechanism across different data ratios using \textrm{Evol-Instruct-Code-80K}. %models and datasets
Removing the weighting results in consistent performance degradation, highlighting its contribution to stable learning.}
\label{tab:ablation_weight_appendix}
\resizebox{1\columnwidth}{!}{
\begin{NiceTabular}{l c ccc}
\toprule
\textbf{Data Ratio} & \textbf{Method} & \textbf{HumanEval(+)} & \textbf{MBPP(+)} & \textbf{Avg.} \\
% \midrule
% \multicolumn{5}{c}{\textbf{CodeLlama-Python-7B} fine-tuned on \textbf{Evol-Instruct-Code-80K}} \\
\midrule
\multirow{2}{*}{1.00 (68\%)} &
w/o weighting  & 47.3(41.5) & 56.1(46.9) & 48.0 \\
& \textbf{w/ weighting} & \textbf{50.2(44.5)} 
             & \textbf{62.5(52.1)} & \textbf{52.3} \\
\midrule
\multirow{2}{*}{0.75 (51\%)} &
w/o weighting  & 47.2(40.1) & 61.8(50.2) & 49.8 \\
& \textbf{w/ weighting} & \textbf{48.8(43.5)} 
             & \textbf{62.6(52.0)} & \textbf{51.7} \\
% \midrule
% \multicolumn{5}{c}{\textbf{CodeLlama-Python-7B} fine-tuned on \textbf{Magicoder-OSS-Instruct-75K}} \\
% \midrule
% \multirow{2}{*}{1.00 (79\%)} &
% w/ weighting & \textbf{54.9(49.8)} & \textbf{67.0(55.6)} & \textbf{56.8} \\
% & w/o weighting  & 52.3(47.6) & 66.8(55.3) & 55.5 \\
% \midrule
% \multirow{2}{*}{0.75 (59\%)} &
% w/ weighting & \textbf{54.5(49.4)} & \textbf{66.9(55.3)} & \textbf{56.5} \\
% & w/o weighting  & 52.4(47.4) & 66.3(55.2) & 55.3 \\
\bottomrule
\end{NiceTabular}
}
\end{table}
\subsection{Effect of Weight Scaling}
In Table~\ref{tab:ablation_weight_appendix}, we present an additional ablation study on the weighting mechanism across different data ratios using \textrm{Evol-Instruct-Code-80K}. 
We compare the full method with a variant that removes the weighting scheme during training.
As shown in the table, applying the weighting consistently yields better performance on all evaluation metrics. 
Removing the weighting results in noticeable performance drops across all data ratios, indicating that the weighting mechanism plays an important role in stabilizing training and preserving informative samples during data selection.

\begin{table}[t]
\centering
\caption{Effect of data selection mechanisms on \textrm{Evol-Instruct-Code-80K}.}
\label{tab:ablation_prune_appendix}
\resizebox{1\columnwidth}{!}{
\begin{tabular}{l l ccc}
\toprule
\textbf{Data Ratio} & \textbf{Selection Metric} & \textbf{HumanEval(+)} & \textbf{MBPP(+)} & \textbf{Avg.} \\
% \midrule
% \multicolumn{5}{c}{\textbf{CodeLlama-Python-7B} fine-tuned on \textbf{Evol-Instruct-Code-80K}} \\
\midrule
\multirow{3}{*}{1.00 (68\%)} &
Random & 47.2(42.9) & 62.7(53.1) & 51.5 \\
& Core-centric  & 45.7(41.9) & 62.1(51.8) & 50.4 \\
& Peripheral  & \textbf{50.2(44.6)} & \textbf{62.5(52.1)} & \textbf{52.3} \\
% \midrule
% \multicolumn{5}{c}{\textbf{CodeLlama-Python-7B} fine-tuned on \textbf{Magicoder-OSS-Instruct-75K}} \\
% \midrule
% \multirow{3}{*}{1.00 (79\%)} &
% Random & 53.1(47.8) & 66.3(55.1) & 55.6 \\
% & Core-centric  & 54.5(49.0) & 66.9(55.0) & 56.4 \\
% & Peripheral  & \textbf{54.9(49.8)} 
%              & \textbf{67.0(55.6)} 
%              & \textbf{56.8} \\
\bottomrule
\end{tabular}}
% \vspace{-0.3cm}
\end{table}
\subsection{Effect of Data Selection Metric}
Table~\ref{tab:ablation_prune_appendix} reports the performance of different within-cluster selection metrics.
Among the compared strategies, peripheral selection consistently achieves the best results, outperforming both random and core-centric selection.
This indicates that prioritizing boundary samples yields more informative training signals once cluster-level diversity is already ensured.

\begin{table}[t]
\centering
\caption{Effect of batch size on \textrm{Magicoder-OSS-Instruct-75K}.}
\label{tab:bs_appendix}
\resizebox{1\columnwidth}{!}{
\begin{tabular}{l c ccc}
\toprule
\textbf{Data Ratio} & \textbf{Batch Size} & \textbf{HumanEval(+)} & \textbf{MBPP(+)} & \textbf{Avg.} \\
% \midrule
% \multicolumn{5}{c}{\textbf{CodeLlama-Python-7B} fine-tuned on \textbf{Evol-Instruct-Code-80K}} \\
\midrule
\multirow{2}{*}{1.00 (79\%)}
& 256 & 54.3 (50.0) & 66.8 (55.1) & 56.5 \\
& 512 & 54.9 (49.8) & 67.0 (55.6) & 56.8 \\
\bottomrule
\end{tabular}}
% \vspace{-0.3cm}
\end{table}
\subsection{Effect of Batch Size}
Table~\ref{tab:bs_appendix} reports the sensitivity analysis with different global batch sizes, while fixing the selection ratio to 1.00. 
\algname{} shows stable performance across batch-size settings and consistently outperforms the random baseline. 
These results indicate that \algname{} is not tied to a particular batch size and remains effective under different data-parallel training configurations.

\begin{table}[t]
\centering
\caption{
Performance of sampling strategies for \textrm{Llama-2-13B} on \textrm{MedInstruct-52K}.}
\label{tab:app_13b}
\resizebox{\linewidth}{!}{
\begin{tabular}{ll ccccc}
\toprule
\textbf{Sel. Ratio} & \textbf{Sampling} 
& \textbf{MedMCQA} 
& \textbf{MedQA} 
& \textbf{MMLU} 
& \textbf{PubMedQA} 
& \textbf{Avg.} \\
\midrule

\multirow{1}{*}{\makecell[l]{Full (100\%)}}
  & Random
  & 39.1
  & 43.1
  & \textbf{55.2}
  & \textbf{75.2}
  & \textbf{53.1} \\
\midrule

\multirow{2}{*}{\makecell[l]{1.00 (65.9\%)}}
  & Random
  & 37.0
  & 41.0
  & \textbf{55.0}
  & 73.6
  & 51.7 \\

  & \cellcolor{clustergray}\textbf{\algname{}}
  & \cellcolor{clustergray}\textbf{39.0}
  & \cellcolor{clustergray}\textbf{42.0}
  & \cellcolor{clustergray}54.5
  & \cellcolor{clustergray}\textbf{75.0}
  & \cellcolor{clustergray}\textbf{52.6} \\
\bottomrule
\end{tabular}}
\vspace{-0.3cm}
\end{table}
\subsection{Additional 13B-Scale Evaluation}
\label{appendix:additional_13b}

To further evaluate \algname{} at the 13B scale across different domains, we fine-tune \textrm{Llama-2-13B} on \textrm{MedInstruct-52K}.
We compare \algname{} with full-data training and random sampling under the same effective data budget as \algname{}.
The results are reported in Table~\ref{tab:app_13b}.

At an effective data ratio of 65.9\%, \algname{} achieves an average score of 52.6, outperforming random sampling (51.7) and approaching full-data training (53.1).
In particular, \algname{} achieves performance close to full-data training on MedMCQA and PubMedQA while using substantially fewer training samples.
Together with the CodeLlama-Python-13B experiment in the main text, these results provide additional evidence that the effectiveness of \algname{} extends across distinct 13B instruction-tuning settings and domains.

\subsection{Representative Examples by Cluster}
\label{appendix:semantic_clusters}
To evaluate the semantic coherence of the discovered clusters, we randomly sampled 400 instances from the dataset and analyzed their cluster assignments. For each cluster, we manually inspected the samples to identify the dominant task category and measured the proportion of samples consistent with this category, providing a quantitative estimate of cluster-level semantic alignment.

\begin{table}[t]
\centering
\caption{Manual evaluation of semantic alignment between clusters and task categories.}
\label{tab:cluster_alignment}
\resizebox{1\columnwidth}{!}{
\begin{tabular}{lccc}
\toprule
Cluster & Category & Matches / Total & Agreement \\
\midrule
0 & Algorithm \& Data Structure & 70 / 86 & 81.4\% \\
1 & SQL \& Front-End & 73 / 78 & 93.6\% \\
2 & Mathematical \& Pattern Matching & 110 / 142 & 77.5\% \\
3 & Basic Programming Grammar & 75 / 94 & 79.8\% \\
\midrule
\textbf{Overall} & -- & \textbf{328 / 400} & \textbf{82.0\%} \\
\bottomrule
\end{tabular}
}
\end{table}

\subsubsection{Statistical Analysis of Cluster Semantics.}
Table~\ref{tab:cluster_alignment} presents a manual evaluation of the semantic alignment between the clusters obtained from gradient-based clustering and their dominant task categories.
Each cluster largely corresponds to a distinct task type, such as algorithms, SQL/front-end generation, mathematical reasoning, or basic programming grammar.
Overall, 328 out of 400 samples (82.0\%) match the dominant category of their cluster, indicating that the gradient-based clustering captures meaningful task-level semantics.

\subsubsection{Representative Examples by Cluster}
To illustrate the characteristics of each cluster, we provide representative examples sampled from the \textrm{Evol-Instruct-Code-80K} dataset.  
Gradient embeddings were extracted from the \textrm{CodeLlama-7B} model, and clustering was performed using $k$-means clustering with $k=4$.  
Since each embedding jointly captures both semantic and uncertainty information, the resulting groups can be interpreted as \textit{semantic clusters} that reflect distinct instructional intents and reasoning patterns.  
Each example lists only the \textit{Instruction} for brevity; corresponding outputs are omitted for simplicity.

% ---- Cluster 0 ----
\paragraph{Cluster 0: Algorithm and Data Structure.}
\vspace{-0.5em}
Representative examples involving structural reasoning, data traversal, and algorithmic implementation.

\begin{tcolorbox}[colback=gray!5, colframe=black!40, title=Cluster 0 Examples, fonttitle=\bfseries, breakable]
\begin{enumerate}[leftmargin=1.2em]
    \item Create a class in Python that represents a binary tree with insert, search, delete, and traversal methods.
    \item Provide instructions to create a doubly linked list in Java with methods to add and remove nodes.
    \item Implement a recursive function to compute the depth of a binary tree.
    % \item Create a function that merges two sorted linked lists into a single sorted list.
    \item Write a program that performs Principal Component Analysis (PCA) on a dataset containing numerical variables.
\end{enumerate}
\end{tcolorbox}

% ---- Cluster 1 ----
\paragraph{Cluster 1: SQL and Front-End Interaction.}
\vspace{-0.5em}
Representative examples involving database queries, UI interactions, and event-driven logic.

\begin{tcolorbox}[colback=gray!5, colframe=black!40, title=Cluster 1 Examples, fonttitle=\bfseries, breakable]
\begin{enumerate}[leftmargin=1.2em]
    \item Create a SQL query to show all employee details with their attachments in a specific department.
    \item Write a JavaScript code snippet to add an event listener to a button that validates an input field.
    \item Introduce a condition to change the background color of a given HTML page to green when the user clicks a button.
    \item Write an AngularJS controller to set the “page” variable from URL parameters and validate it.
\end{enumerate}
\end{tcolorbox}

% ---- Cluster 2 ----
\paragraph{Cluster 2: Mathematical and Pattern Matching.}
\vspace{-0.5em}
Representative examples of numerical computation, recursive rules, and pattern-based logic.

\begin{tcolorbox}[colback=gray!5, colframe=black!40, title=Cluster 2 Examples, fonttitle=\bfseries, breakable]
\begin{enumerate}[leftmargin=1.2em]
    \item Extend the given code to print numbers from 1 to 11 incremented by 2 and sum those divisible by 3.
    \item Write an algorithm to print the first 100 prime numbers without using built-in functions.
    \item Write a Python function to add two matrices of the same dimensions without using NumPy.
    \item Write a function to check if the sum of digits of a number is prime.
    % \item Implement a recursive algorithm with memoization to find the nth Fibonacci number.
\end{enumerate}
\end{tcolorbox}

% ---- Cluster 3 ----
\paragraph{Cluster 3: Basic Programming Grammar.}
\vspace{-0.5em}
Representative examples covering fundamental syntax, control flow, and basic data manipulation.

\begin{tcolorbox}[colback=gray!5, colframe=black!40, title=Cluster 3 Examples, fonttitle=\bfseries, breakable]
\begin{enumerate}[leftmargin=1.2em]
    \item Change the following array so that it orders in ascending order.
    \item Replace the given code with an equivalent ternary expression without logical operators.
    \item Write a code that takes a number and prints it out in words.
    % \item Create a Python function that takes two inputs and prints whether they are equal, greater, or smaller.
    \item Print numbers from 1 to 10 using a while loop.
\end{enumerate}
\end{tcolorbox}
% \section{Example Appendix}
\label{sec:appendix}

\end{document}